\PassOptionsToPackage{hypertexnames=false}{hyperref}
\documentclass{theoretics}

\usepackage{bbm}

\usepackage{cleveref}
\newtheorem*{meta*}{Guiding Principle}
\newcommand{\R}{\mathbb{R}}

\newcommand{\bigo}{\mathcal{O}}

\newcommand{\N}{\mathbb{N}}

\newcommand{\Ebb}{\mathbb{E}}

\newcommand{\algo}{\mathcal{A}}
\newcommand{\dist}{\mathscr{D}}
\newcommand{\Rplus}{\mathbb{R}_{\geq 0}}
\newcommand{\unif}{\mathrm{unif}}

\LinesNotNumbered 
\Crefname{claim}{Claim}{Claims} 

\title{Realizable Learning is All You Need}

\ThCSauthor[ucsd-cs]{Max Hopkins}{nmhopkin@ucsd.edu}
\ThCSauthor[ucsd-cs,ucsd-math]{Daniel M. Kane}{dakane@eng.ucsd.edu}
\ThCSauthor[ucsd-cs]{Shachar Lovett}{slovett@cs.ucsd.edu}
\ThCSauthor[yale]{Gaurav Mahajan}{gaurav.mahajan@yale.edu}

\ThCSaffil[ucsd-cs]{Department of Computer Science and Engineering, UCSD, California, CA 92092}
\ThCSaffil[ucsd-math]{Department of Mathematics, UCSD, California, CA 92092}
\ThCSaffil[yale]{Department of Computer Science, Yale University, Connecticut, CT 06511}

\ThCSshorttitle{Realizable Learning is All You Need}

\ThCSyear{2024}
\ThCSarticlenum{2}
\ThCSreceived{Sep 28, 2022}
\ThCSrevised{Nov 21, 2023}
\ThCSaccepted{Dec 11, 2023}
\ThCSpublished{Feb 3, 2024}
\ThCSkeywords{PAC-Learning, Realizable Learning, Agnostic Learning, Blackbox Reductions}
\ThCSdoi{10.46298/theoretics.24.2}
\ThCSshortnames{M. Hopkins, D. M. Kane, S. Lovett and G. Mahajan}
\ThCSthanks{Invited article from COLT 2022.}

\begin{document}
\maketitle

\begin{abstract}
The equivalence of realizable and agnostic learnability is a fundamental phenomenon in learning theory. With variants ranging from classical settings like PAC learning and regression to recent trends such as adversarially robust learning, it's surprising that we still lack a unified theory; traditional proofs of the equivalence tend to be disparate, and rely on strong model-specific assumptions like uniform convergence and sample compression. 

In this work, we give the first model-independent framework explaining the equivalence of realizable and agnostic learnability: a three-line blackbox reduction that simplifies, unifies, and extends our understanding across a wide variety of settings. This includes models with no known characterization of learnability such as learning with arbitrary distributional assumptions and more general loss functions, as well as a host of other popular settings such as robust learning, partial learning, fair learning, and the statistical query model.

More generally, we argue that the equivalence of realizable and agnostic learning is actually a special case of a broader phenomenon we call property generalization: any desirable property of a learning algorithm (e.g.\ noise tolerance, privacy, stability) that can be satisfied over finite hypothesis classes extends (possibly in some variation) to any learnable hypothesis class. 
\end{abstract}

\section{Introduction}\label{sec:intro}
The equivalence of realizable and agnostic learnability in Valiant's Probably Approximately Correct (PAC) model \cite{valiant1984theory} is one of the best known results in learning theory, and numbers among its most surprising. Given a set $X$ and a family of binary classifiers $H$, the result states that the ability to learn a classifier $h \in H$ from examples of the form $(x,h(x))$ is in fact sufficient for something much stronger: given samples from \textit{any} distribution $D$ over $X \times \{0,1\}$, it is possible to learn the \textit{best approximation to $D$ in $H$}. This surprising equivalence stems from a classical result of Vapnik and Chervonenkis (VC) \cite{vapnik1974theory}, and independently Blumer, Ehrenfeucht, Haussler, and Warmuth (BEHW) \cite{Blumer} and Haussler \cite{haussler1992decision}, who equate both the former model (known as \textit{realizable learning}) and the latter model (known as \textit{agnostic learning}) to a strong property of pairs $(X,H)$ called \textit{uniform convergence}.\footnote{Uniform convergence promises that a large enough sample gives a good approximation for loss of \textit{every} $h \in H$ simultaneously.}

VC, BEHW, and Haussler's result was certainly a breakthrough in its own right, but its proof technique is too indirect to reveal any deeper connections between realizable and agnostic learning beyond the PAC setting. Further, recent years have seen both theory and practice shift away not only from this original formalization, but more generally from the ``uniform convergence equals learnability'' paradigm, often in favor of distributional or data-dependent assumptions like margin that are more applicable to the real world. The inability of VC, BEHW, and Haussler's proof technique to generalize to such scenarios raises a fundamental question: \textit{is the equivalence of realizable and agnostic learning a fundamental property of learnability, or simply a happy coincidence derived from the original PAC framework?}

In the 30 years since these works, a mountain of evidence has amassed in favor of the former: almost every reasonable variant of learning shares some sort of similar equivalence. This includes a long list of popular settings such as regression \cite{bartlett1996fat}, distribution-dependent learning \cite{benedek1991learnability}, multi-class learning \cite{david2016supervised}, robust learning \cite{montasser2019vc}, online learning \cite{ben2009agnostic}, private learning \cite{beimel2014learning, alon2020closure}, and partial learning \cite{long2001agnostic,alon2021theory}. What's more, the uniform convergence paradigm fails miserably in most of these models. In the distribution-dependent model, for instance, it is easy to build classes which are trivially learnable (even with one sample!) but completely fail to satisfy uniform convergence \cite{benedek1991learnability}. On the other hand, models such as private learning give well-known examples where uniform convergence fails to imply learnability \cite{alon2019private}. In spite of this, we are really no closer today to a general understanding of this phenomenon than we were in the early 90s. Much like Vapnik and Chervonenkis \cite{vapnik1974theory}, Blumer, Ehrenfeucht, Haussler, and Warmuth \cite{Blumer}, and Haussler's \cite{haussler1992decision} proofs, the above works often use indirect methods and tend to rely on powerful model-dependent assumptions. 

In this work, we aim to offer a generic, unifying theory by way of the first direct reduction from agnostic to realizable learning. Unlike any previous work, our reduction is blackbox, relies on no additional assumptions, and, perhaps most importantly, is incredibly simple. In fact, the basic algorithm can be stated in three lines. 

\begin{algorithm}[H]
\SetAlgoLined
\textbf{Input:} Realizable PAC-Learner $\algo$, Unlabeled Sample Oracle $\bigo_U$, Labeled Sample Oracle $\bigo_L$
\\
\textbf{Algorithm:}
\begin{enumerate}[leftmargin=*]
    \item Draw an unlabeled sample $S_U \sim \bigo_U$, and labeled sample $S_L \sim \bigo_L$.
    \item Run $\algo$ over all possible labelings of $S_U$ to get:
    \[
    C(S_U) \coloneqq \{ \algo(S_U,h(S_U))~|~ h \in H|_{S_U}\}.
    \]
    \item \textbf{Return} the hypothesis in $C(S_U)$ with lowest empirical error over $S_L$.
\end{enumerate}
 \caption{Agnostic to Realizable Reduction}
 \label{Intro:alg}
\end{algorithm}

This basic reduction simplifies and unifies classic results such as VC \cite{vapnik1974theory}, BEHW \cite{Blumer}, and Haussler's \cite{haussler1992decision} distribution-free equivalence and Benedek and Itai's \cite{benedek1991learnability}
analogous result in the distribution-dependent setting\footnote{Benedek and Itai only consider random classification noise, but it is clear that their analysis extends to the agnostic model.} up to log factors. More importantly, because \Cref{Intro:alg} doesn't rely on model-dependent properties like uniform convergence, it extends to learning regimes without known characterizations. One such example is the notoriously difficult distribution-family model, in which the adversary is given a restricted family of distributions $\dist$ along with the pair $(X,H)$. The distribution-family model has no finitary characterization \cite{lechner2023impossibility},\footnote{In an earlier version of this manuscript we conjectured this result. It has since been confirmed by Lechner and Ben-David \cite{lechner2023impossibility}.} yet \Cref{Intro:alg} can still be used to show that the realizable and agnostic settings are equivalent.

Unfortunately, while \Cref{Intro:alg} does avoid any significant blowup in sample complexity, it is inherently computationally inefficient. In fact, this is necessary unless $P=NP$. There are many basic classes (e.g.\ halfspaces) which are easy to learn in the realizable model, but NP or cryptographically hard in the agnostic setting (see e.g.\ \cite{feldman2009agnostic}). As such we focus in this work only on information theoretic considerations, though building computationally efficient reductions in restricted settings remains an interesting avenue of research.

The core of \Cref{Intro:alg} lies in an equivalence between PAC learning and a new form of randomized covering of independent interest we call a \textit{non-uniform cover}. In contrast to more classical notions, a non-uniform cover is a distribution over subsets of hypotheses that covers any \textit{fixed} hypothesis in the class with high probability, but may fail to cover all hypotheses \textit{simultaneously}. The connection between supervised learning and non-uniform covering is inherent in \Cref{Intro:alg}, where Steps 1 and 2 turn the realizable learner $\mathcal{A}$ into a non-uniform cover $C(S_U)$, and Step 3 uses the cover to perform agnostic learning. At a high level, this process works because the adversary does not see the randomness inherent to Steps 1 and 2, and therefore cannot detect or exploit which hypotheses in the class will fail to be well-estimated in the process.

In fact, this connection has many broader implications within the theory of supervised learning. For one, the method is not inherently restricted to the agnostic setting. \Cref{Intro:alg} achieves agnostic learning for general classes by applying an agnostic learner for finite classes in Step 3; replacing this with learners satisfying other properties (e.g.\ the exponential mechanism for privacy) leads to other families of reductions to the realizable setting. At a high level, this can be summarized by the following informal `guiding principle':
\begin{meta*}[Property Generalization]\label{thm:meta}
    If there is a (sample-efficient) algorithm with property $P$ over finite classes, then \Cref{Intro:alg} gives a (sample-efficient) learner with property $P$ over any `learnable' class.
\end{meta*}
We stress that the above is a guide, not a theorem, and indeed often requires modification or weakening of the desired property for a given application. For instance, when Step 3 is replaced with a private algorithm, the result is a \textit{semi}-private learner for general classes, a weakened model allowing the use of a small amount of public unlabeled data \cite{beimel2013private,alon2019limits}.

\Cref{Intro:alg} also provides a unified framework for many settings beyond the PAC-model. This includes basic extensions such as general loss functions\footnote{Over infinite label spaces, we will require some weak assumptions on the loss.} and the distribution-family model, but also more involved modifications such as partial learning, robust learning, or even the statistical query model. Moreover, in some of these settings removing reliance on setting-specific assumptions like uniform convergence actually quantitatively improves the sample complexity; such is the case for the semi-private model, where we use this fact to achieve information-theoretically optimal unlabeled sample complexity for the first time. 

Finally, we note there are a few settings where \Cref{Intro:alg} runs into issues, especially discrete infinite settings such as infinite multi-class classification and properties such as privacy that require more careful data handling. We leave the extension of our method to these settings as an intriguing open problem.
\subsection{Proof Overview}
Before moving to more detailed discussion of our results, we briefly overview the proof that \Cref{Intro:alg} is an agnostic learner. At its most basic, the idea is simply to observe that the set $C(S_U)$ almost certainly contains a `near-optimal' hypothesis $\tilde{h} \in H$. Since $C(S_U)$ has bounded size, standard uniform convergence for finite classes promises Step 3 outputs a hypothesis close to $\tilde{h}$, which is therefore itself `near-optimal' as desired.

Taking a step back, the key observation in this process is really that $C(S_U)$ `acts like a cover' of $H$ in the following weak sense we call \textit{non-uniform covering}:
\begin{definition}[Non-uniform Cover (Informal \Cref{def:prob-cover})]
Let $(X,H)$ be a hypothesis class, $D$ a marginal distribution over $X$, and $C$ a random variable over the power set $P(H)$. We call $C$ a non-uniform $(\varepsilon,\delta)$-cover of $H$ with respect to $D$ if:
\[
\forall h \in H: \quad \Pr_{C}\left[\exists h' \in C: \Pr_{x \sim D}[h'(x) \neq h(x)] \leq \varepsilon\right] \geq 1-\delta.
\]
\end{definition}
In other words, for every \textit{fixed} hypothesis $h$ in the class, $C$ is very likely to contain a hypothesis close to $h$. This differs from the standard covering which enforces that $C$ \textit{simultaneously} cover every $h \in H$ (equivalent to pushing the `$\forall$' quantifier into the probability above). In \Cref{sec:cover}, we show that the latter (which is the standard in the literature) is strictly stronger and requires more samples to generate. This is critical in our application to semi-private learning where we cut the additional samples required to generate a full cover and thereby achieve the optimal sample complexity.

It is left to observe that $C(S_U)$ in \Cref{Intro:alg} is actually a non-uniform cover, but this is essentially immediate from the definition of realizable learning! Realizable learning promises that for every fixed hypothesis $h \in H$, when $A$ receives samples labeled by $h$ it outputs a hypothesis close to $h$ with high probability. $C(S_U)$ is generated by running $A$ across all $h \in H$, so this guarantee exactly translates to the above. We refer the reader to \Cref{sec:overview-base,sec:base-reduction} for a formal explanation of this argument.

\subsection{Beyond PAC Learning}
\Cref{Intro:alg} is an extremely flexible framework for proving agnostic to realizable reductions in supervised learning. In this section, we informally overview the many extended models studied in this paper. The most basic setting in which \Cref{Intro:alg} applies beyond the standard model is \textit{learning under distributional assumptions} (or formally what we call the ``distribution-family model'', see \Cref{sec:overview-base,sec:prelims}). Standard techniques using combinatorial dimensions require that the learner $A$ works over \textit{every distribution} (the ``distribution-free'' setting), while modern algorithms frequently only work under some niceness conditions on the distribution. On the other hand, it is easy to observe that the process described above works under arbitrary distributional assumptions---as long as $A$ is a realizable learner over a distribution $D$, then \Cref{Intro:alg} is an agnostic learner whenever the data is marginally distributed as $D$. This leads to the following corollary:

\begin{theorem}[Distribution Family Model (Informal \Cref{thm:basic-reduction})] 
The sample complexity of agnostic learning a class $(X,H)$ in the distribution family model is at most:
\[
m(\varepsilon,\delta) \leq O\left(\frac{n(\varepsilon,\delta)+\log(1/\delta)}{\varepsilon^2}\right),
\]
where $n(\varepsilon,\delta)$ is the realizable sample complexity of $(X,H)$.
\end{theorem}
In the finite VC setting, this can be improved to roughly $\frac{d\log(\frac{d}{\varepsilon})+\log(\frac{1}{\delta})}{\varepsilon^2}$ for VC dimension $d$, which is optimal up to a log factor.

\Cref{Intro:alg} also covers many other supervised settings in the literature that diverge from the PAC model in more substantial ways, including general loss functions, constraints on the learner, and properties beyond agnostic learning. For simplicity we cover these here informally, avoiding exact definitions and sample complexities (which are generally similar to the above), and give formal details in the main body.

\subsubsection{General Loss Functions} 

Perhaps the most natural extension of the PAC framework is to loss functions beyond binary classification. Here data points are labeled by an arbitrary \textit{label space} $Y$ and error is measured with respect to a generic \textit{loss function} $\ell: Y \times Y \to \mathbb{R}_{\geq 0}$, for instance we may take $Y = \R$ and $\ell(y,y')=(y-y')^2$, the square loss. In general, agnostic and realizable learning are actually not equivalent in this setting, even for nice loss functions (see \Cref{prop:infinite-lower}). The issue is that over infinite label spaces it is possible to encode hypotheses into the labels with infinite precision, making the class trivial to learn realizably, but impossible even with the smallest amount of noise (which wipes out the encoding). We show this type of counter-example is essentially the \textit{only barrier} to agnostic learnability. 

Somewhat more formally, we call a class $H$ \emph{discretely learnable} if for every $\varepsilon > 0$, there exists an $\varepsilon$-discretization\footnote{A discretization is roughly a class $H'$ over a (probably) finite label space that $\varepsilon$-covers $H$. See \Cref{sec:infinite} for details.} of $H$ that is learnable up to $O(\varepsilon)$ error. Discrete learnability can informally be thought of as a very weak type of noise tolerance that essentially acts only to rule out the above construction. We prove discrete and agnostic learnability are equivalent under weak conditions on the loss function.
\begin{theorem}[General Loss (Informal \Cref{thm:bounded,thm:approximate pseudometric})]\label{intro:infinite}
If $\ell$ is a (probably) upper bounded loss function\footnote{Here we mean that it is enough for $\ell$ to be bounded with good probability over the data distribution.} and either:
\begin{enumerate}
    \item $\ell$ is bounded from below
    \item $\ell$ satisfies a $C$-approximate triangle inequality
\end{enumerate}
then discrete and agnostic learnability are equivalent under $\ell$ up to polynomial factors.
\end{theorem}
We remark that in the latter case, the learner only achieves $C\cdot OPT+\varepsilon$ error and that this is tight (\Cref{prop:c-agn-lower}). To our knowledge, these results are new even to the distribution-free setting, where such an equivalence was only known for bounded Lipschitz \cite{bartlett1996fat,wolf2018mathematical} or binary-valued \cite{ben1995finiteY,david2016supervised} loss functions. In the models below, similar guarantees hold under the above assumptions on loss. We omit the exact dependence on $\ell$ for simplicity and discuss where relevant in the main body.



Another well-studied setting for loss functions beyond standard classification is \textit{adversarial robustness}. Robust learning is an extension of the PAC model introduced to handle adversarial perturbations at test time by taking a \textit{maximum} of the loss function over specified `perturbation sets' around the test point. We give a modification of \Cref{Intro:alg} in the robust setting that handles general loss and distributional assumptions.
\begin{theorem}[Robust Classification (Informal \Cref{thm:robust})]\label{intro:robust}
    Robust realizable and robust agnostic learning are equivalent up to polynomial factors.
\end{theorem}
Note this is completely independent of the perturbation sets. In the classification setting, \Cref{intro:robust} generalizes recent work giving such an equivalence in the distribution-free model \cite{montasser2019vc,montasser2021adversarially}, though the sample complexity of our algorithm suffers an extra factor of $\varepsilon^{-1}$ in this special case.

Finally, another recent setting that works with a sort of `modified' loss function is the learning of \textit{partial} functions, capturing `data-dependent' models such as halfspaces with geometric margin where. Here the functions in $H$ are allowed to be `undefined' on some portions of $X$, and the loss is fixed to $1$ under any undefined point no matter the response of the algorithm.
\begin{theorem}[Partial functions (Informal \Cref{thm:partial})]
    Realizable and agnostic learning of partial functions are equivalent up to polynomial factors.
\end{theorem}
In the distribution-free setting, a variant of this equivalence is known via compression schemes, but the above is new under distributional assumptions and more general loss (again at the cost of an extra $\varepsilon^{-1}$ factor).

\subsubsection{Constrained Models}

Another frequent modification of the PAC setting is to impose constraints either on how the algorithm $A$ uses the data, or on various properties of the output classifier itself. In this section, we cover application to three such examples: fairness, stability, and statistical queries.

We start with \textit{fairness}, where the goal is to output a hypothesis with low error \textit{conditional} on `treating similar individuals similarly' under a fixed metric on the data space. 
\begin{theorem}[Fair Learning (Informal \Cref{thm:fair})]
 Realizable and agnostic fair learning are equivalent up to polynomial factors.
\end{theorem}
Formally, the above result holds in Rothblum and Yona's \cite{rothblum2018probably} popular `Probably Approximately Correct and Fair' (PACF)-learning. To our knowledge it is the first such reduction even in the distribution-free setting, as Rothblum and Yona \cite{rothblum2018probably} study the agnostic case directly.

Another popular constraint in the literature is \textit{algorithmic stability}, where one roughly imposes that running the algorithm twice should produce similar results. We study a simple variant known as \textit{uniform stability} \cite{dagan2020pac} (closely related to `private prediction' \cite{dwork2018privacy}) which promises that for every fixed point in the space $x$ the output distribution of $A$ on $x$ is similar when $A$ is run on neighboring datasets.
\begin{theorem}[Uniform Stability (Informal \Cref{thm:stable})]\label{thm:intro-stability}
    Realizable and agnostic learning are equivalent under uniform stability up to polynomial factors.
\end{theorem}
This is known in the distribution-free setting by VC arguments \cite{dwork2018privacy,dagan2020pac}, but is new for general loss and under distributional assumptions. Unlike prior examples the sample complexity matches prior techniques.

Another way to constrain the algorithm $A$ is through the way it interacts with data. Perhaps the most popular example of this is the \textit{statistical query model}, where $A$ is constrained to approximating general population statistics (such algorithms are then guaranteed to satisfy certain nice properties such as noise-tolerance and privacy). We give a variant of our reduction for the statistical query setting as well.
\begin{theorem}[Statistical Query Model (Informal \Cref{thm:SQ})]
    Realizable and agnostic learning are equivalent in the statistical query model up to polynomial factors.
\end{theorem}
Variants of such a result are known via combinatorial characterization in the fixed distribution setting \cite{simon2007characterization,feldman2009agnostic}, but new in our general setup.

\subsubsection{Beyond Agnostic Learning}

In the introduction, we claimed \Cref{Intro:alg} can be used to build a learner satisfying \textit{any} ``finitely-satisfiable'' property, not just agnostic learning. In fact, one of the examples above already displays this property: \Cref{thm:intro-stability} does not require the base learner to be uniformly stable. Instead, we apply a uniformly stable agnostic learner to $C(S_U)$ in Step 3 of \Cref{Intro:alg} and `lift' uniform stability from finite to infinite classes. We give two further applications of this strategy: to malicious noise, and (semi)-private learning.


Kearns and Li's \cite{kearns1993learning} \textit{malicious noise} is a model for data contamination where one is given a faulty sample oracle $\bigo_M(\cdot)$ which (with some small probability) returns a completely adversarial pair $(x,y)$, and otherwise draws from a true `ground truth' distribution. A learner is said to be tolerant to malicious noise if it achieves standard PAC guarantees while drawing from $\bigo_M(\cdot)$. Like agnostic learning, tolerance to malicious noise is known to be achievable for finite classes \cite{kearns1993learning}. As a result, (a slight modification of) \Cref{Intro:alg} gives a blackbox reduction from learning with malicious noise to realizable learning.
\begin{theorem}[Realizable $\to$ Malicious (Informal \Cref{thm:malicious})]\label{intro:malicious}
    Realizable learning and learning with malicious noise are equivalent up to polynomial factors.
\end{theorem}
This extends the original result of \cite{kearns1993learning} to the distribution-family and general loss settings.

Our final application of \Cref{Intro:alg} is to the ubiquitous notion of \textit{differential privacy}. Informally, an algorithm is said to be differentially private if its output is not susceptible to small changes in the underlying sample (see \Cref{sec:privacy}). Privacy is a strong condition, and even relaxed notions require finite Littlestone dimension in the distribution-free setting \cite{alon2019private}, ruling out a direct reduction from realizable learning. However, \Cref{Intro:alg} does recover a weaker variant known as \textit{semi-privacy} \cite{beimel2013private} where the learner is allowed to use a few `public' unlabeled samples, but must otherwise maintain privacy with respect to the data. 
\begin{theorem}[Realizable $\to$ Semi-Private (Informal \Cref{thm:private-PAC})]\label{intro:semi-private}
Realizable and semi-private learning are equivalent up to polynomial factors.
\end{theorem}
This generalizes Beimel, Nissim, and Stemmer's \cite{beimel2013private} equivalence in the distribution-free setting and extends to general loss. Perhaps most interesting is that in this case, due to relying only on \textit{non-uniform covering} over more classical uniform convergence arguments, \Cref{Intro:alg} actually gives gives a quantitative improvement over prior methods.
\begin{corollary}[Semi-Optimal Semi-Private Learning (Informal \Cref{cor:private-class})]\label{intro:cor-privacy}
    Let $(X, H)$ be a class with VC-dimension $d$. Then $(X,H)$ is $\alpha$-semi-private, agnostically learnable in
    \[
    m_{\text{pub}}(\varepsilon,\delta,\alpha) \leq O\left(\frac{d+\log(1/\delta)}{\varepsilon}\right)
    \]
    unlabeled (public) samples, and
    \[
    m_{\text{pri}}(\varepsilon,\delta,\alpha) \leq O\left(\frac{d\log(1/\varepsilon) + \log(1/\delta)}{\varepsilon \cdot \min\{\varepsilon,\alpha\}}\right)
    \]
    labeled (private) samples.
\end{corollary}
For fixed $d$ and $\delta$, the public sample complexity of \Cref{intro:cor-privacy} is tight and resolves a conjecture of Alon, Bassily, and Moran \cite{alon2019limits} who gave the corresponding lower bound.\footnote{Specifically, any class with infinite Littlestone dimension requires $\Omega(\frac{1}{\varepsilon})$ public unlabeled samples to semi-privately learn.} The private complexity is off by a log factor from known lower bounds \cite{chaudhuri2011sample}, and it remains an interesting question whether this can be fixed.

\Cref{intro:semi-private} and \Cref{intro:cor-privacy} are also robust to light distribution shift between the public and private databases. This problem, called \textit{covariate shift}, is a commonly observed in machine learning practice and is especially of concern in privacy where a distribution over ``opt-in'' public users could easily differ from the overall distribution of private data. We discuss covariate shift in more depth in \Cref{sec:covariate}.
\subsection{Modification Archetypes}
Finally before moving to formal presentation of our methods and results, we briefly overview the four generic modifications of \Cref{Intro:alg} used in the above extensions, and outline where they can be found formally in the main body of the paper.

\paragraph{Discretization.} We'll start by discussing our main technique to extend \Cref{Intro:alg} to infinite label spaces. The basic idea is simple: since we cannot afford to run our learner over all possible labelings of $S_U$, we instead run the learner over labelings coming from some \textit{discretization} of the class.
As long as we have access to a learner for the discretization, we can then use the same arguments covered in \Cref{sec:overview-base} to prove various occurrences of property generalization. We formalize these notions more generally in \Cref{sec:infinite}, where we use the technique to prove \Cref{intro:infinite}. Discretization can also be used to handle learning models such as the statistical query setting which output real-valued query responses (see \Cref{sec:SQ}).

\paragraph{Subsampling.} Another core limitation to \Cref{Intro:alg} is access to clean unlabeled data. \Cref{Intro:alg} works by running a realizable learner over a representative set of unlabeled data, but, in practice, such data may often be corrupted, and data-dependent assumptions such as margin might mean that the optimal hypothesis isn't even well-defined on this set. We handle cases like these by a simple sub-sampling procedure: instead of running our realizable learner over labelings of $S_U$, we run the learner over all labelings of \textit{all subsets of $S_U$}. As long as $S_U$ contains some amount of uncorrupted data, this subsampling procedure will find it and we can maintain the guarantees discussed in \Cref{sec:overview-base}. We use this technique to prove property generalization for models such as robust learning (see \Cref{intro:robust} and \Cref{sec:robust}), partial learning (see \Cref{sec:partial}), and malicious noise (see \Cref{intro:malicious} and \Cref{sec:subsample}).

\paragraph{Replacing the Finite Learner.} In the introduction, we proposed a general paradigm (guiding principle) called property generalization: that a variant of any learning property which holds for finite classes should in fact hold for \textit{any} ``learnable'' class in the base model. The main idea relies on replacing Step 3 of \Cref{Intro:alg} (which, as stated, is an empirical risk minimization process) with a generic learner for finite classes with the desired property. For noise-tolerance properties such as agnostic and malicious noise, empirical risk minimization works. Properties such as privacy or stability, however, require a different finite learner. To prove \Cref{intro:semi-private}, for example, we replace the ERM process in \Cref{Intro:alg} with the exponential mechanism \cite{mcsherry2007mechanism}. We use a similar strategy in \Cref{sec:stable} to prove an analogous result for uniform stability.

\paragraph{Replacing the Base Learner.} Finally we note a very basic modification of \Cref{Intro:alg} that allows us to extend property generalization beyond the PAC setting: simply replace the input realizable PAC learner with a realizable learner in the desired model. This is usually combined with one of the techniques above depending on the specific application, e.g. to prove property generalization for robust learning and the statistical query model. The same idea can also be used to analyze semi-private learning with covariate shift (see \Cref{sec:covariate}) and property generalization for fair learning (see \Cref{sec:fair}).

\section{Distribution Family Classification}\label{sec:overview-base}
Since all of our results are derived from variants of \Cref{Intro:alg}, it is instructive to start by considering its basic analysis in our simplest non-trivial setting: distribution-family classification. We remark that this section is entirely pedalogical, and the results are subsumed in \Cref{sec:base-reduction}.

The distribution-family model captures learnability with arbitrary distributional assumptions, a well-studied relaxation of PAC learning in practice where worst-case distributional assumptions are often too strong, and encompasses both the distribution-free and distribution-dependent PAC settings. Unlike these models, however, we cannot assume uniform convergence in the distribution family setting, due to a classical example of Benedek and Itai \cite{benedek1991learnability}.
\begin{proposition}[Benedek and Itai \cite{benedek1991learnability}]
    \label{prop:uc-bc}
There exists a PAC-learnable class $(D,X,H)$ over binary labels and classification loss without the uniform convergence property.
\end{proposition}
\begin{proof}
Let $X=[0,1]$, $D$ be the uniform distribution over $X$, $Y=\{0,1\}$, and $H$ consist of all indicator functions for finite sets $S \subset X$, as well as for $X$ itself. It is not hard to see that $(D,X,H)$ is realizably PAC-learnable by the following scheme in only a single sample: if the learner draws a sample labeled $1$, output the all $1$'s function. Otherwise, output all $0$s. When the adversary has chosen a finite set, with probability $1$ the learner draws a sample labeled $0$, and outputs a hypothesis with $0$ error (since the finite set has measure $0$). If the adversary chooses the all $1$'s function, the learner will always output the all $1$'s function.

On the other hand, it is clear that when the adversary chooses the all $1$'s function, no matter how many samples the learner draws, there will exist a hypothesis in the class that is poorly approximated by the sample. Namely the hypothesis whose support is given by the support of the sample itself has empirical measure $1$, but true measure $0$. As a result, this class fails to have the uniform convergence property despite its learnability.
\end{proof}
This simple example rules out many classical techniques in supervised learning, including uniform convergence based reductions (e.g. \cite{balcan2010discriminative,beimel2013private,alon2019limits}). Furthermore, since the distribution-family setting has no finitary characterization, we also cannot hope to take the traditional approach \cite{vapnik1974theory,benedek1991learnability,haussler1992decision,bartlett1996fat} of proving equivalence of realizable and agnostic models by simultaneous characterization.

With this motivation in mind, let's define distribution-family learning more formally. Let $X$ be a set (called the \textit{instance space}), $Y=\{0,1\}$ the set of binary labels, $\dist$ a family of distributions over $X$, and $H = \{h: X \to Y\}$ a family of binary classifiers. A tuple $(\dist,X,H)$ is said to be \textit{realizably learnable} if there exists an algorithm\footnote{Note that $\algo$ could be deterministic or randomized. This distinction has no effect on any of the arguments in this work.} $\algo$ and a function $n(\varepsilon,\delta)$ such that for every $\varepsilon,\delta > 0$, every choice of distribution $D \in \dist$, and every hypothesis $h \in H$, $\algo$ outputs a good classifier with high probability on samples of size $n(\varepsilon,\delta)$:
\[
\Pr_{S \sim D^{n(\varepsilon,\delta)}}[\text{err}_{D \times h}(\algo(S,h(S))) \leq \varepsilon] \geq 1-\delta,
\]
where $\text{err}_{D\times h}(\cdot)$ is commonly called the \textit{error} or \textit{risk} of $h$:
\[
err_{D\times h}(h') = \underset{x \sim D}{\Pr}[h'(x) \neq h(x)].
\]
Likewise, a tuple $(\dist,X,H)$ is said to be \textit{agnostically learnable} if there exists an algorithm $\algo$ which for every distribution $D$ over  $X \times Y$ whose marginal $D_X \in \dist$ outputs $h'$ close to the best hypothesis in $H$ with probability $1-\delta$:
\[
\Pr_{S \sim D^{n(\varepsilon,\delta)}}[\text{err}_{D}(\algo(S)) \leq OPT + \varepsilon] \geq 1-\delta,
\]
where $OPT=\inf_{h \in H}\{\text{err}_D(h)\}$ is the error of the best hypothesis in the class\footnote{We will assume throughout for simplicity that the infinum is realized by some element in the class, but this assumption can be removed with essentially no changes to any arguments.} and the risk $\text{err}_D(\cdot)$ is similarly defined:
\[
\text{err}_D(h') = \underset{(x,y) \sim D}{\Pr}[h'(x) \neq y].
\]
With this in mind, we can now state the most basic application of \Cref{Intro:alg}: the equivalence of agnostic and realizable learning for distribution-family classification.

\begin{theorem}[Realizable $\to$ Agnostic (Distribution-Family Classification)]\label{thm:intro-classification}
    Let $\algo$ be a realizable learner for $(\dist,X,H)$ using $n(\varepsilon,\delta)$ samples. Then \Cref{Intro:alg} is an agnostic learner for $(\dist,X,H)$ using:
    \[
    m_U(\varepsilon,\delta) \leq n(\varepsilon/2,\delta/2)
    \]
    unlabeled samples, and
    \[
    m_L(\varepsilon,\delta) \leq O\left(\frac{n(\varepsilon/2,\delta/2)+\log(1/\delta)}{\varepsilon^2}\right)
    \]
    labeled samples. Moreover if $(X,H)$ has finite VC dimension $d$, \Cref{Intro:alg} needs only
    \[
    m_L(\varepsilon,\delta) \leq O\left(\frac{d\log\left(1/\varepsilon\right)+\log(1/\delta)}{\varepsilon^2}\right)
    \]
    labeled samples.
\end{theorem}
Along with its novelty in the distribution-family setting (where no such equivalence was known), it is worth noting that in the distribution-free setting, \Cref{thm:intro-classification} actually recovers the same sample complexity bound as standard analysis of uniform convergence (though it remains off by a log factor from the optimal bound given by chaining \cite{li2001improved}). We also note that while unlabeled sample complexity is not usually considered separately from labeled complexity in the PAC setting, this will become a useful distinction in semi-supervised extensions considered later in the work. As such, it is instructive to keep the complexities separate for the time being.

With this out of the way, let's prove \Cref{thm:intro-classification}. The analysis breaks naturally into two parts, corresponding respectively to Step 2 and Step 3 of \Cref{Intro:alg}. In the first part, we'll show that $C(S_U)$, the set of outputs corresponding to running the realizable learner $\algo$ across all possible labelings of the unlabeled sample $S_U$, is in some sense a ``good approximation'' of the class $H$. More formally, the crucial observation is that for any choice of the adversary's distribution, $C(S_U)$ will (almost) always contain a hypothesis close to the optimal solution.
\begin{claim}\label{pf-overview-claim}
For any distribution $D$ over $X \times Y$ whose marginal $D_X \in \dist$, with probability $1-\delta/2$, there exists $h' \in C(S_U)$ which is within $\varepsilon/2$ of the optimal risk:
\[
err_{D}(h') \leq OPT + \varepsilon/2.
\]
\end{claim}
Once we have this claim, the second step is to show that Step 3, an empirical risk minimization process on $C(S_U)$, gives the desired agnostic learner. This actually follows from standard arguments. In particular, given a hypothesis $h \in C(S_U)$, let
\[
err_{S_L}(h) = \Pr_{(x,y) \sim S_L}[h(x) \neq y]
\]
denote its \textit{empirical risk} with respect to $S_L$. Since $C(S_U)$ is finite, a standard Chernoff+Union bound gives that with probability at least $1-\delta/2$, the empirical risk of every hypothesis in $C(S_U)$ with respect to $S_L$ is close to its true risk. Then as long as $S_L$ is sufficiently large, empirical risk minimization 
returns a solution with at most $OPT + \varepsilon$ error with high probability (we'll formalize this in a moment).

It remains to prove \Cref{pf-overview-claim}. The key observation lies in an equivalence between realizable PAC-learning and a weak type of randomized covering: for any fixed $h \in H$, $C(S_U)$ contains a hypothesis close to $h$ with high probability.
\begin{lemma}\label{intro:lemma-cover}
For any distribution $D$ over $X \times Y$ with marginal $D_X \in \dist$ and any $h \in H$, with probability $1-\delta/2$, there exists $h' \in C(S_U)$ which is within $\varepsilon/2$ of $h$ in classification distance:
\[
\Pr_{x \sim D_X}[h'(x) \neq h(x)] \leq \varepsilon/2.
\]
\end{lemma}
\begin{proof}
The proof is essentially immediate from the definition of realizable PAC-learning. $\algo$ promises that for any $h \in H$ and $D \in \dist$, a $1-\delta/2$ fraction of \textit{labeled} samples $(S,h(S)) \sim D^{n(\varepsilon/2,\delta/2)}$ satisfy 
\[
\text{err}_{D\times h}[\algo(S,h(S))] = \underset{D}{\Pr}[h'(x) \neq h(x)] \leq \varepsilon/2,
\]
where $h'=\algo(S,h(S))$. Since $C(S_U)$ contains $\algo(S_U,h(S_U))$ for every $h \in H$ by definition, the result follows.
\end{proof}
More generally, we call such objects \textit{non-uniform covers}. 
\begin{definition}[Non-uniform Cover (Informal \Cref{def:prob-cover})]
Let $(X,H)$ be a class over label space $Y$, $D$ a marginal distribution over $X$, and $C$ a random variable over the power set $P(H)$. We call $C$ a non-uniform $(\varepsilon,\delta)$-cover of $H$ with respect to $D$ if for every $h \in H$:
\[
\Pr_{C}\left[\exists h' \in C: \Pr_{x \sim D}[h'(x) \neq h(x)] \leq \varepsilon\right] \geq 1-\delta.
\]
\end{definition}
Note that \Cref{intro:lemma-cover} (and in general non-uniform covering) does not imply that $C(S_U)$ contains hypotheses close to every $h \in H$ \textit{simultaneously}. This stronger object is called a \textit{uniform} cover and takes provably more samples to construct (see \Cref{sec:cover}). In our case, a non-uniform cover is sufficient. Since the guarantee holds for every \textit{fixed} $h \in H$, it must hold in particular for the optimal hypothesis $h_{OPT}$, so $C(S_U)$ contains some $h'$ within $\varepsilon/2$ of optimal. Let's now formalize these ideas and put everything together to prove \Cref{thm:intro-classification}.
\begin{proof}[Proof of Theorem \ref{thm:intro-classification}] 
Let $D$ be the adversary's distribution over $X \times Y$, and let $h_{OPT} \in H$ be a hypothesis achieving the optimal error. By \Cref{intro:lemma-cover}, with probability $1-\delta/2$, $C(S_U)$ contains a hypothesis $h'$ such that:
\[
\Pr_{x \sim D_X}[h'(x) \ne h_{OPT}(x)] \leq \varepsilon/2.
\]
This implies \Cref{pf-overview-claim} (that $C(S_U)$ contains a hypothesis with error at most $OPT+\varepsilon/2$) since
\begin{align*}
    err_D(h') &= \underset{(x,y) \sim D}{\Pr}[h'(x) \neq y]\\
    &\leq \underset{(x,y) \sim D}{\Pr}[h_{OPT}(x) \neq y] + \underset{(x,y) \sim D}{\Pr}[h_{OPT}(x) \neq h'(x)]\\
    &\leq OPT + \varepsilon/2.
\end{align*}
We can now use standard empirical risk minimization bounds on $C(S_U)$ to find a hypothesis with error at most $OPT+\varepsilon$. Chernoff and union bounds imply that with probability at least $1-\delta/2$, the empirical risk of every hypothesis in $C(S_U)$ on a sample of size $O\left(\frac{\log(|C(S_U)|/\delta)}{\varepsilon^2}\right)$ is at most $\varepsilon/4$ away from its true error. Since $h'$ has error at most $OPT+\varepsilon/2$, its empirical risk is at most $OPT+3\varepsilon/4$, and by the above guarantee any hypothesis in $C(S_U)$ with empirical risk at most $OPT+3\varepsilon/4$ has true error at most $OPT + \varepsilon$. 

Putting everything together, we have that with probability $1-\delta$ over the entire process, the empirical risk minimizer of $C(S_U)$ has error at most $OPT+\varepsilon$ as desired. The sample complexity bounds follow from noting that $|C(S_U)|$ is at most $2^{n(\varepsilon/2,\delta/2)}$, and at most $\left(\frac{e\cdot n(\varepsilon/2,\delta/2)}{d}\right)^d$ if the class has VC dimension $d$. The sample complexity bound for the latter case then follows by plugging in the standard bound for distribution-free classification: $n(\varepsilon/2,\delta/2) \leq O\left( \frac{d\log(1/\varepsilon) + \log(1/\delta)}{\varepsilon}\right)$.
\end{proof}

\section{Related Work}\label{sec:related-work}
Agnostic learning is a very widely studied model across learning theory, and works across many different sub-areas have noted model specific equivalences with realizable learning. Here we'll survey a few representative examples, and discuss how they relate and differ from our approach.

\subsection{Beyond Binary Classification}
\paragraph{Uniform Convergence and Multiclass Classification.} 

It is well known that the uniform convergence equals learnability paradigm continues to hold for $0/1$-valued loss functions over constant-size label spaces \cite{Blumer,ben1995finiteY,vapnik1971,natarajan1989}, and that agnostic and realizable learning are equivalent as a result. On the other hand, Daniely, Sabato, Ben-Devid, and Shalev-Shwartz \cite{daniely2015} showed this is no longer the case as the number of labels grows large. In this regime, even basic multi-class learning is no longer equivalent to uniform convergence, so the connection between realizable and agnostic learning becomes non-trivial. A few years later, David, Moran, and Yehuyadoff (DMY) \cite{david2016supervised} proved the equivalence nevertheless holds in the infinite multi-class setting through the weaker \textit{sample compression equals learnability} paradigm. While more general than the uniform convergence paradigm, their proof remains model-specific and fails in many of the settings we consider, e.g.\ partial learning \cite{alon2021theory}.

\paragraph{Discretization and General Loss Functions.} Basic forms of discretization were also considered back in the mid 90s in work on characterizing the learnability of real-valued functions. In a seminal work, Bartlett, Long, and Williamson (BLW) \cite{bartlett1996fat} proved that a scale-sensitive measure introduced by Kearns and Schapire \cite{Kearns93} called fat-shattering dimension characterizes learnability under bounded Lipschitz loss functions.\footnote{Though the original work only considers $\ell1$ loss, their techniques generalize to Lipschitz loss, see for instance \cite{wolf2018mathematical}.} BLW use a basic form of discrete learning (called quantization) to prove that fat-shattering dimension is a necessary condition, and use uniform convergence to prove sufficiency. 
We give a similar argument as BLW in the necessary direction, but show that uniform convergence is not necessary for the equivalence to hold, and instead use \Cref{Intro:alg} to appeal directly to discrete learnability. This allows us to extend BLW's result across a much more general set of loss functions and scenarios without strong model-specific assumptions.

\subsection{Semi-supervised, Active, and Semi-Private Learning}
Our reduction hinges on combining a realizable learner with unlabeled data to cut down the number of potential hypotheses in our class. The use of unlabeled samples to this effect is one of the core ideas in the field of \textit{semi-supervised} learning \cite{balcan2010discriminative,Zhu2017}. Here, it is usually additionally assumed that the function to be learnt has some relation (or `compatibility') to the underlying data distribution, for example it might have large margin on unlabelled data as in Transductive SVM \cite{joachims1999transductive} or redundant sufficient information as in Co-training \cite{blum1998combining,dasgupta2001}. In their seminal work on the topic, Balcan and Blum \cite{balcan2010discriminative} employed a similar strategy to \Cref{Intro:alg} in which they draw an unlabeled sample $S_U$ and select hypotheses consistent with each possible labeling based upon compatibility. They argue via uniform convergence that this results in a uniform cover, and then use empirical risk minimization to select a good hypothesis in the cover. It is worth noting that around the same time a similar strategy independently found use in the online learning literature in work of Ben-David, Pal, and Shalev-Schwarz \cite{ben2009agnostic} who simulated the so-called `standard optimal algorithm' (SOA) over a sequence of examples and applied weighted majority \cite{littlestone1994weighted} over the resulting set of hypotheses to obtain an agnostic online learner.

Similar strategies have also found use in the related \textit{active learning} literature. Hanneke and Yang \cite{hanneke2015minimax} use the same technique to build a cover from unlabeled samples (adding one hypothesis consistent with each possible labeling), and then apply active (adaptive) query algorithms to learn the best hypothesis in the cover in as few labeled samples as possible. This generalized earlier work of Dasgupta \cite{dasgupta2005coarse}, who assumed a priori that the cover was known to the learner ahead of time. Most recently, the approach has seen use in the study of semi-private learning. In their original work on the model, Beimel, Nissim, and Stemmer \cite{beimel2013private} again apply the same trick for building a uniform cover, but then find the best hypothesis privately via the exponential mechanism (similar to our proof of \Cref{intro:semi-private}). The analysis of this strategy was later improved by Alon, Bassily, and Moran (ABM) \cite{alon2019limits}.

The above works differ from ours in two crucial senses. First, each work focuses solely on developing an algorithm for their specific framework (rather than working to understand a more general equivalence or reduction between settings). In this sense, one can view each of these prior results as a specific instance of our general framework where the ``base learner'' in our reduction is restricted to be an empirical risk minimizer (or SOA in the online setting), and a problem-specific learner for the relevant property (online, agnostic, active, or private) is then applied over the resulting cover. Second, and perhaps most importantly, these previous works all rely fundamentally on uniform convergence.\footnote{The online setting is the only exception, but this model diverges much more substantially from the PAC-type learning we consider in this work than previous uniform convergence based methods.} This means that their algorithms break down as soon as one moves away from the original PAC model (even to say the basic distribution-dependent setting), and can also lead to sub-optimal sample complexity bounds. In the analysis of semi-private learning, for instance, we show that avoiding  uniform convergence leads to asymptotically better bounds, actually resolving the public sample complexity of the model altogether. Indeed, one can show that building a uniform cover requires asymptotically more unlabeled samples than a non-uniform one, and therefore cannot result in optimal semi-supervised algorithms.\footnote{Formally we prove this for the distribution-family model, but conjecture it holds in the distribution-free setting too.} 

\subsection{Non-Uniform Covering and Probabilistic Representations}
Covering techniques have long been used in learning theory, and while almost all prior works focus on \textit{uniform} notions (where all hypotheses are covered simultaneously), there is one notable exception. In 2013, Beimel, Nissim, and Stemmer (BNS) \cite{beimel2013characterizing} introduced \textit{probabilistic representations}, a strong randomized form of covering used to characterize pure differentially private learning. In the language of our work, given a class $(X,H)$, a probabilistic representation is a distribution over subsets of $H$ which is a non-uniform cover \textit{simultaneously over all distributions of $X$}. BNS prove that private learning is equivalent to the existence of a probabilistic representations for the class. Equivalently, this can be thought of as the ability to build a non-uniform cover without access to the underlying distribution at all. On the other hand, we are interested in the much weaker setting where a non-uniform cover can be built from a bounded number of samples from the distribution (and crucially argue that this is equivalent to realizable learning). Thus in a sense, our core connection between realizable learning and non-uniform covering can be thought of as an analog of BNS' characterization of private learning by probabilistic representations. 


\subsection*{Paper Organization}
The main body of this paper is split into two main portions. In the first we cover preliminaries (\Cref{sec:prelims}), the base reduction for all finitely supported loss functions (\Cref{sec:base-reduction}), and discuss in detail the four main modification archetypes to \Cref{Intro:alg}: discretization (\Cref{sec:infinite}), sub-sampling (\Cref{sec:subsample}), replacing ERM (\Cref{sec:privacy}), and changing the base model (\Cref{sec:covariate}), covering extensions to infinite label classes, malicious noise, semi-private learning, and covariate shift respectively. In the remainder of the paper we cover applications of these modified versions to doubly-bounded loss (\Cref{sec:doubly-bounded}), robust learning (\Cref{sec:robust}), partial learning (\Cref{sec:partial}), uniformly-stable learning (\Cref{sec:stable}), the statistical query model (\Cref{sec:SQ}), and fair learning (\Cref{sec:fair}), and discuss further connections of non-uniform covers to previous notions of covering (\Cref{sec:cover}). Each of the latter sections are self-contained and the reader is encouraged to skip directly to any model they wish to see directly.

\section{Preliminaries}\label{sec:prelims}
Before moving to a more formal discussion of our results, we'll cover the most basic learning models discussed in this work: standard (distribution-free) PAC-learning and distribution-family PAC-learning. Extended models we consider beyond these (e.g. malicious noise, robust learning, partial learning, etc.) will instead be introduced in their respective sections.

\subsection{PAC-Learning}
We'll start by reviewing the seminal PAC-learning model of Valiant \cite{valiant1984theory} and Vapnik and Chervonenkis \cite{vapnik1974theory}. We start with a few core definitions for the setting of general loss. Let $X$ be an arbitrary set called the \emph{instance space} (e.g. $\mathbb{R}^d$), $Y$ a set called the \emph{label space} (e.g. $\{0,1\}$), and $H$ a family of labelings of $X$ by $Y$ (that is a family of functions of the form $h: X \to Y$). Given a class $(X,H)$, it will often be useful to consider its \textit{growth function} $\Pi_H(n)$ which measures the maximum size of $H$ when restricted to a sample of size $n$:
\[
\Pi_H(n) = \max_{h \in H, S \in X^n}(\left|H|_S\right|)).
\]
We note that the growth function is trivially bounded by $|Y|^n$, but one can often give stronger bounds when $(X,H)$ satisfies some finite combinatorial dimension (e.g.\ VC-dimension for the binary case).

While PAC-learning is sometimes used to refer only to classification, we will study the model under general loss functions. With that in mind, we call a function $\ell: Y \times Y \to \Rplus$ a \emph{loss function} if $\ell(y,y)=0$ for all $y \in Y$. We say a loss $\ell$ satisfies the \emph{identity of indiscernibles} if $\ell(y_1,y_2) = 0$ iff $y_1=y_2$. Given any distribution $D$ over $X \times Y$ and loss $\ell$, the \emph{risk} of a labeling $h: X \to Y$ with respect to $D$ and $\ell$ is its expected loss:
\[
\text{err}_{D,\ell}(h) = \underset{(x,y) \sim D}{\mathbb{E}}[\ell(h(x),y)].
\]
The goal of learning is generally to find a classifier $h \in H$ that minimizes risk. More formally, there are two commonly studied variants of this problem. The original formulation, now called \emph{realizable} learning, assumes the existence of a hypothesis in $H$ with no loss.
\begin{definition}[(Realizable) PAC-learning]
We say $(X, H, \ell)$ is realizable PAC-learnable if there exists an algorithm $\algo$ and function $n(\varepsilon,\delta)$ such that for all $\varepsilon,\delta > 0$ and distributions $D$ over $X\times Y$ such that $\min_{h \in H}{\text{err}_{D,\ell}(h)}=0$:
\[
\Pr_{S \sim D^{n(\varepsilon,\delta)}}[\text{err}_{D,\ell}(\algo(S)) > \varepsilon] \leq \delta.
\]
$\algo$ is called \textbf{proper} if it outputs only labels in $H$.
\end{definition}
Perhaps a more realistic variant of PAC-learning is to drop this restriction on the adversary, and let them choose an arbitrary distribution over $X \times Y$. This model, introduced by Haussler \cite{haussler1992decision} and Kearns, Schapire, and Sellie \cite{kearns1993learning}, is known as \emph{agnostic} learning.
\begin{definition}[(Agnostic) PAC-learning]
We say $(X, H, \ell)$ is agnostic PAC-learnable if there exists an algorithm $\algo$ and function $n(\varepsilon,\delta)$ such that for all $\varepsilon,\delta > 0$ and distributions $D$ over $X \times Y$:
\[
\Pr_{S \sim D^{n(\varepsilon,\delta)}}[\text{err}_{D,\ell}(\algo(S)) > OPT + \varepsilon] \leq \delta,
\]
where $OPT = \min_{h \in H}\{\text{err}_{D,\ell}(h)\}$.
\end{definition}
For some settings covered in this work, it will turn out that reaching $OPT+\varepsilon$ error is too stringent of a condition. However, we will show in these cases that it is sometimes possible to maintain a weaker guarantee and learn up to $c \cdot OPT + \varepsilon$ error for some constant $c>1$. We call such classes \emph{$c$-agnostic learnable}.

Finally, we note that for simplicity when $\ell$ is the standard ``classification error:''
\[
\ell(y_1,y_2) = \begin{cases}
0 & \text{if } y_1=y_2\\
1 & \text{else},
\end{cases}
\]
we'll simply write $(X,H)$ to mean $(X,H,\ell)$. Realizable and Agnostic Learning are well studied under many basic loss functions including binary classification, where both models are known to be characterized by a combinatorial parameter called VC-dimension.

\subsection{Learning Under Distribution Families}
\label{sec:def-arbitrary-distribution}
The standard PAC-models described above are often called \emph{distribution-free} due to the fact that no assumptions are made on the marginal distribution over $X$. In practice, however, this is usually too worst-case an assumption. We often expect distributions in nature to be ``nice'' in some way, or at least somewhat restricted. This is reflected in the fact that popular machine learning algorithms usually significantly outperform the PAC-model's worst-case generalization bounds. Indeed such niceness assumptions have long been popular in learning theory as well, where conditions such as tail bounds or anti-concentration are frequently used to build efficient algorithms. 

These ideas are captured more generally by a simple (but notoriously difficult) extension to the PAC framework originally proposed by Benedek and Itai \cite{benedek1991learnability}, where the adversary is restricted to picking from a fixed, known set of distributions.
\begin{definition}[(Realizable) Distribution-Family PAC-learning]
Let $X$ be an instance space and $\dist$ a family of distributions over $X$. We say $(\dist, X, H, \ell)$ is realizable PAC-learnable if there exists an algorithm $\algo$ and function $n(\varepsilon,\delta)$ such that for all $\varepsilon,\delta > 0$ and distributions $D$ over $X\times Y$ satisfying:
\begin{enumerate}
    \item The marginal $D_X \in \dist$,
    \item $\min_{h \in H} \text{err}_{D,\ell}(h) = 0$,
\end{enumerate}
we have
\[
\Pr_{S \sim D^{n(\varepsilon,\delta)}}[\text{err}_{D,\ell}(\algo(S)) > \varepsilon] \leq \delta.
\]
\end{definition}
Agnostic learning is defined similarly. The adversary must still choose a marginal distribution in $\dist$, but the conditional labeling can be arbitrary.
\begin{definition}[Agnostic Distribution-Family PAC-learning]
We say $(\dist, X, H, \ell)$ is agnostic PAC-learnable if there exists an algorithm $\algo$ and function $n(\varepsilon,\delta)$ such that for all $\varepsilon,\delta > 0$ and distributions $D$ over $X \times Y$ satisfying:
\begin{enumerate}
    \item The marginal $D_X \in \dist$,
\end{enumerate}
$\algo$ outputs a good hypothesis with high probability:
\[
\Pr_{S \sim D^{n(\varepsilon,\delta)}}[\text{err}_{D,\ell}(\algo(S)) > OPT + \varepsilon] \leq \delta
\]
where $OPT = \min_{h \in H}\{\text{err}_{D,\ell}(h)\}$.
\end{definition}
The weaker $c$-agnostic learning is defined analogously with $OPT$ replaced by $c \cdot OPT$. Unlike the standard model, very little is known about distribution-family learnability. A number of works have given partial characterizations of the model \cite{benedek1991learnability,dudley1994metric,kulkarni1997learning,vidyasagar2001closure}, but it was recently shown by Lechner and Ben-David \cite{lechner2023impossibility} that no general finitary characterization is possible.

\section{The Core Reduction: Finite Label Classes}\label{sec:base-reduction}
In this section, we give a more detailed exposition of our main reduction as covered in \Cref{sec:overview-base} for the general setting of arbitrary loss on constant size label spaces, matching lower bounds, and additional discussion of non-uniform covers. As mentioned previously, since there is no combinatorial characterization of learnability in the distribution-family model \cite{lechner2023impossibility}, standard techniques \cite{Blumer,ben1995finiteY,vapnik1971,natarajan1989} cannot be used.

Before jumping into our reduction proper, it is worth re-iterating why we can't simply take the approach of prior works and rely on \textit{uniform convergence}, a strong condition which promises that on a large enough sample, the empirical error of every hypothesis will be close to its true error. While uniform convergence was a very popular technique in the early years of learning, practitioners have since moved away from the paradigm which fails to capture learning rates seen in practice \cite{zhang2021,nagarajan2019}. Indeed it soon became clear that the technique failed to capture even basic theoretical models such as the distribution-dependent setting (as discussed in \Cref{prop:uc-bc}). In later sections, we will even see distribution-free models where uniform convergence fails, such as the Partial PAC model \cite{long2001agnostic,alon2021theory} which captures realistic scenarios such as learning with margin. Since even the most basic modifications of PAC-learning fail to satisfy uniform convergence, it is clear we need to move beyond the condition to gain a more general understanding of the common phenomenon of equivalence between learning models.

Instead of relying on uniform convergence, our core observation is an equivalence between learning and sample access to a combinatorial object we call a \textit{non-uniform cover}.
\begin{definition}[Non-uniform Cover]\label{def:prob-cover}
Let $(X,H)$ be a class over label space $Y$ and $L_{X,Y}$ denote the family of all labelings from $X$ to $Y$. If $C$ is a random variable over the power set $P(L_{X,Y})$ and $d: L_{X,Y} \times L_{X,Y} \to \Rplus$ is a ``distance'' function between labelings, we call $C$ a non-uniform $(\varepsilon,\delta)$-cover of $H$ with respect to $d$ if for all $h \in H$:
\[
\Pr_{T \sim C}\left[\exists h' \in T: d(h', h) \leq \varepsilon\right] \geq 1-\delta.
\]
We call $C$ \textbf{bounded} if its support lies entirely on subsets of size at most some $k \in \mathbb{N}$, and we call the smallest such $k$ its \textbf{size}.
\end{definition}
Non-uniform covers share a close connection to several notions of covering used throughout the learning literature such as uniform covers \cite{alon2019limits} and fractional covers \cite{alon2021adversarial}. We discuss these connections in more detail in \Cref{sec:cover}. For the moment, we note only that previous works using the strictly stronger notion of uniform covering necessarily lose factors in the sample complexity as a result. We discuss this further in \Cref{sec:privacy}.

In \Cref{sec:overview-base}, we argued (at least implicitly) that once we have sampling access to a bounded non-uniform cover, agnostic learnability follows from standard arguments. Namely since a sample $T$ has bounded size and is guaranteed to contain a concept ``close'' to optimal, it suffices to run empirical risk minimization over about $\log(|T|/\delta)/\varepsilon^2$ samples. The key to our reduction therefore boils down to turning blackbox access to a realizable PAC-learner into sampling access to some relevant non-uniform cover. This is given by Step 2 of \Cref{Intro:alg}, which we rewrite here as a subroutine called \textsc{LearningToCover}.

\begin{algorithm}[H]
\SetAlgoLined
\textbf{Input:} Hypothesis Class $H$, Realizable PAC-Learner $\algo$, Unlabeled Sample size $S_U$.
\\
\textbf{Algorithm:}
\begin{enumerate}[leftmargin=*]
    \item Run $\algo$ over all possible labelings of $S_U$ by $H$.
    \item \textbf{Return} the set of responses $C(S_U)$:
    \[
    C(S_U) \coloneqq \{ \algo(S_U,h(S_U))~|~ h \in H|_{S_U}\}.
    \]
\end{enumerate}
 \caption{\textsc{LearningToCover}($H$, $\algo$, $S_U$)}
 \label{alg:learningtocover}
\end{algorithm}

In fact, we already argued in \Cref{sec:overview-base} that \textsc{LearningToCover} gives sampling access to a non-uniform cover, but we will re-write the result here in this formulation for convenience.
\begin{lemma}[Core Lemma: Realizable Learning Implies Non-Uniform Covering]\label{lemma:cover}
Let $\algo$ be an algorithm that $(\varepsilon,\delta)$-PAC learns a class $(\dist, X, H, \ell)$ in $n=n(\varepsilon,\delta)$ samples. Then for any $D \in \dist$, running \textsc{LearningToCover} on $S_U \sim D^n$ returns a sample from a size $\Pi_H(n)$, non-uniform $(\varepsilon,\delta)$-cover with respect to the standard distance between hypotheses:
\[
d(h,h') = \underset{D}{\mathbb{E}}[\ell(h'(x),h(x))].
\]
\end{lemma}
\begin{proof}
The proof is essentially immediate from the definition of realizable PAC-learning. $\algo$ promises that for any $h \in H$ and $D \in \dist$, a $1-\delta$ fraction of \textit{labeled} samples $(S_U,h(S_U)) \sim D^{n(\varepsilon,\delta)}$ satisfy 
\[
\text{err}_{D\times h,\ell}(\algo(S_U,h(S_U))) = \underset{D}{\mathbb{E}}[\ell(h'(x),h(x))] \leq \varepsilon,
\]
where $h'=\algo(S_U,h(S_U))$. Since $C(S_U)$ contains $\algo(S_U,h(S_U))$ for every $h$ by definition, the result follows.
\end{proof}
This means that as long as we have blackbox access to a realizable PAC-learner and unlabeled samples from the adversary's distribution, we can simulate access to a non-uniform cover. Let's now formalize our previous intuition that this is sufficient to turn a realizable learner into an agnostic one for any finite label class. We will generalize this result to doubly-bounded loss in \Cref{sec:doubly-bounded}, but it is instructive to consider the setting of finite $Y$ first.
\begin{theorem}[Realizable $\to$ Agnostic (Finite Label Classes)]\label{thm:basic-reduction}
    Let $(\dist, X, H, \ell)$ be any class on a finite label space $Y$ with loss function $\ell: Y \to Y$ satisfying the identity of indiscernibles. Then \Cref{Intro:alg} is an agnostic learner with sample complexity:
    \[
    m(\varepsilon,\delta) \leq n(\eta_\ell\varepsilon,\delta/2) + O\left(\frac{\log\left(\frac{\Pi_H(n(\eta_\ell\varepsilon,\delta/2))}{\delta}\right)}{\varepsilon^2}\right)
    \]
    where $\eta_\ell \geq \Omega\left(\frac{\min_{a \neq b}(\ell(a,b))}{\max_{a \neq b}(\ell(a,b))}\right)$ is a constant depending only on $\ell$.
\end{theorem}
\begin{proof}
Let $\algo$ be the promised realizable learner for $(X, H, \dist, \ell)$ with sample complexity $n(\varepsilon,\delta)$. Run \textsc{LearningToCover} with parameters $\varepsilon'=\eta_\ell\varepsilon$ and $\delta'=\delta/2$. We argue that the output contains some $h'$ such that $\text{err}_{D,\ell}(h') \leq OPT + \varepsilon/2$. Since $C(S_U)$ is finite and $\ell$ is upper bounded, a standard Chernoff bound gives that choosing an empirical risk minimizer from $C(S_U)$ based on $O\left(\frac{\log\left(\frac{|C(S_U)|}{{\delta}}\right)}{\varepsilon^2}\right)$ additional samples gives the desired learner. The sample complexity then follows immediately from the fact that $C(S_U)$ contains one hypothesis for every labeling of the sample, and is therefore bounded by the growth function $\Pi_H(n)$.

To see why $C(S_U)$ has this property, recall that for any $h \in H$, \Cref{lemma:cover} states that $C(S_U)$ contains some $h'$ such that:
\[
\mathbb{E}_{x \sim D_X}[\ell(h'(x),h(x))] \leq \eta_\ell \varepsilon.
\]
Because we assume that $\ell(a,b) = 0$ iff $a=b$, this actually implies a stronger relation---$h$ and $h'$ must be close in \textit{classification} error:
\[
\Pr_{x \sim D_X}[h(x) \neq h'(x)] \leq \frac{\eta_\ell \varepsilon}{\min_{a \neq b}(\ell(a,b))}.
\]
Let $h_{OPT} \in H$ be an optimal hypothesis, and let $h'_{OPT}$ denote its corresponding close hypothesis as promised above in the output of \textsc{LearningToCover}. Then by the above, we have that:
\begin{align*}
    \text{err}_{D,\ell}(h'_{OPT}) &= \underset{(x,y) \sim D}{\mathbb{E}}[\ell(h'_{OPT}(x),y)]\\
    &= \underset{(x,y) \sim D}{\mathbb{E}}[\ell(h'_{OPT}(x),y) - \ell(h_{OPT}(x),y) + \ell(h_{OPT}(x),y)]\\
    &= OPT + \underset{(x,y) \sim D}{\mathbb{E}}[\ell(h'_{OPT}(x),y) - \ell(h_{OPT}(x),y)]\\
    &\leq OPT + Pr_D[h_{OPT}(x) \neq h'_{OPT}(x)]\max_{a \neq b}(\ell(a,b))\\
    &\leq OPT + \varepsilon/2
\end{align*}
where we have used the assumption that we set $\eta_\ell = c\frac{\min_{a \neq b}(\ell(a,b))}{\max_{a \neq b}(\ell(a,b))}$ for some universal $c>0$.
\end{proof}

It's worth spending a moment discussing our only assumption on the loss function $\ell$, that it satisfies the identity of indiscernibles. This is not only a natural assumption for most cases in practice (that mislabeling has non-zero error), it is theoretically justified as well: realizable and agnostic learning aren't necessarily equivalent for $\ell$ without this property, even in the distribution-free setting.
\begin{proposition}[Identity of Indiscernibles Lower Bound]
    \label{prop:lowerbound-agnostic}
There exists a realizably learnable class $(X,H,\ell)$ over a finite label space $Y$ which is not agnostically learnable. 
\end{proposition}
\begin{proof}
        Let the instance space $X = \N$ be the set of natural numbers, the label space $Y = \{0,1\}^2$. We consider the hypothesis class $H$ with all functions which output the first bit as $0$, that is:  
        \begin{align*}
        H = \{h : h(x) = (0, \cdot) ~\forall x\in X\}.
        \end{align*} 
        Furthermore, we define the loss function $\ell: Y \times Y \to \{0,1,c\}$ as 
        \begin{align*}
        \ell((b_1, r_1), (b_2,r_2)) = \begin{cases}
        0 & b_1 = b_2\\
        1 & b_1 \neq b_2 ~\text{and}~ r_1 = r_2\\
        c & ~\text{otherwise.}
        \end{cases}
        \end{align*} 
        Note that $(X, H, \ell)$ is trivially learnable in the realizable setting simply by returning any $h \in H$. On the other hand, we will show it is only $O(c)$-agnostically learnable. First, notice that for any labelling $f: X \to Y$, there exists a hypothesis $h \in H$ which matches $f$ on the second bit, and therefore for any marginal $D$ over $X$: 
        \[
        OPT \leq err_{D,f}(h) \leq 1.
        \]
        As a result, it suffices to show that for every $m \in \mathbb{N}$ and (randomized) algorithm $\algo$ using $m$ samples there exists a labeling $f: X \to Y$ and marginal distribution $D_X$ such that
        \begin{align}\label{eq:expected-error}
        \Ebb_{S \sim D^m_{X}} \Ebb_{x \sim D_{X}} [\ell(\algo(S,f(S))(x), f(x))] \geq c/12.
        \end{align} 
        As long as this holds Markov's inequality gives that every algorithm must have error at least $\Omega(c)$ with constant probability. 
        
        For simplicity, we will restrict our attention in the rest of the proof to the marginal distribution $D_X$ which is uniform over the set $[k]$ for some natural number $k$ we will fix later. To prove \Cref{eq:expected-error}, by Yao's minimax principle it is enough to prove there is a distribution $\mu$ over functions $f : [k] \to Y$ such that any \emph{deterministic} algorithm $\algo$ has expected loss at least $c/12$ over $\mu$: 
        \begin{align*}
         \Ebb_{f\sim \mu} \Ebb_{S \sim D^m_{X}}\Ebb_{x\sim D_X}[\ell(\algo(S,f(S))(x), f(x) )] > c/12.
        \end{align*} 
            We now show that the above holds for $\mu$ being uniform over all functions from $[k]$ to $Y$ for any $k > 2m$. Here, we have that 
        \begin{align*}
            &\Ebb_{x\sim D_X}\Ebb_{f\sim \mu} \Ebb_{S \sim D^m_{X}}[\ell(\algo(S,f(S))(x), f(x)))]\geq \Ebb_{x\sim D_X}\Big[\Pr_{S \sim D^m_{X}}[x \notin S] \cdot c/4\Big],
        \end{align*} 
        where the last step follows from noting that for any value $(a,b)$ that $\algo(S, f(S))$ assigns to $x \notin S$, $f(x)$ will be $(1-a,1-b)$ with probability $1/4$ incurring a loss of $c$.
        The result then follows by noting that for every $x \in [k]$: \begin{align*}
        \Pr_{S \sim D^m_{X}}[x \notin S] = \Big(1-\frac{1}{k}\Big)^m \geq 1/e
        \end{align*} 
        since $1-x \geq \exp\{-x/(1-x)\}$ and $k > 2m$. Therefore, we get that 
        \begin{align*}
            &\Ebb_{x\sim D_X}\Ebb_{f\sim \mu} \Ebb_{S \sim D^m_{X}}[\ell(\algo(S,f(S))(x), f(x))]\geq c/4e,
        \end{align*} 
        which completes the proof.
\end{proof}
Note that this bound holds even if $\algo$ is allowed to be improper. It is worth noting that if we are willing to increase the size of $Y$, the learner's error in this bound can actually be increased all the way to $c$, the maximum possible (see \Cref{prop:infinite-lower}). This is in fact tight, as we will show that any loss function like the above satisfying a $c$-approximate triangle inequality can be $c$-agnostically learned (that is learned to within $c\cdot OPT + \varepsilon$ error).

\section{Four Modification Archetypes}

\subsection{Discretization: Infinite Label Classes}\label{sec:infinite}

In the previous section, we showed that our base reduction characterizes the equivalence of realizable and agnostic learning for loss functions satisfying the identity of indiscernibles for all \textit{finite} label classes. In this section, we discuss a technique called \textit{discretization} that allows our reduction to extend this result to infinite label classes. Normally, it's clear that when $Y$ is infinite our standard reduction will fail: since the total number of possible labelings of a finite sample may be infinite, \textsc{LearningToCover} may output an infinite set. In fact, this is more than a technical barrier: realizable and agnostic learning simply aren't equivalent for infinite label classes.
\begin{proposition}\label{prop:infinite-lower}
Let $\ell$ be any continuous loss function (in the first variable) over $\mathbb{R}$ satisfying the identity of indiscernibles. Then there exists a class $(\dist,X,H,\ell)$ which is realizably learnable but not agnostically learnable.
\end{proposition}
\begin{proof}
Let $X=\mathbb{N}$ and $Y=[0,2]$. To construct our class, we first consider the set of all boolean functions over $X$ with finite support. Each function $f$ in this class may equivalently be thought of as a binary string in $\{0,1\}^{*}$. Denote the corresponding decimal value of this string in $[0,1]$ by $s_f$. To construct $H$, for every boolean function $f:\mathbb{N} \to \{0,1\}$ with finite support, include in $H$ the function $h_f(x) = f(x) + s_f$. Note that $(X,H)$ is clearly realizably learnable under any distribution family $\dist$ and any loss function, since a single sample always uniquely determines $h_f$. On the other hand, adding even the smallest amount of noise erases this unique identification, making the class impossible to learn.

More formally, let $\dist$ be the family of all distributions. By the continuity of $\ell$ and the fact that $\ell(0,0)=\ell(1,1)=0$, for all $\varepsilon > 0$ notice that there exists $\gamma=\gamma(\varepsilon)>0$ such that $\max_{0 \leq \gamma' \leq \gamma}\{\ell(\gamma',0),\ell(1+\gamma',1)\} < \varepsilon$. Let $n_\gamma \in \mathbb{N}$ be the index of the first non-zero digit in the binary representation of $\gamma$. The idea is to note that beyond these first $n_\gamma$ coordinates, our class if within $\varepsilon$ of an arbitrary boolean function. More formally, notice that for any distribution $D$ and boolean function $f$ which is $0$ on $[n_\gamma]$, we have that $\text{OPT}_H(f) \coloneqq \min_{h \in H}\{\text{err}_{D \times f,\ell}(h)\} \leq \varepsilon$. The bound then follows from the fact that such arbitrary functions are not learnable.

In more detail, Yao's minimax principle states that it is sufficient to show that for any potential sample complexity $m(\varepsilon,\delta)$, there exists a randomized strategy for the adversary such that no deterministic learner can achieve $OPT + c$ accuracy with constant probability for some constant $c>0$. To this end, consider the following strategy: the adversary chooses the uniform distribution over $[n_\gamma, n_\gamma + 2m(\varepsilon,\delta)]$, and a binary function on that interval uniformly at random (recall $OPT$ is at most $\varepsilon$ for every such function). Since the learner can only see $1/2$ of the mass, any strategy must be incorrect on half of the remaining points in expectation. In particular conditioned on any sample, the expected loss of any predicted labeling of an unseen point is at least $\ell_{\text{min-err}} = \underset{y \in [0,2]}{\min}\{\frac{\ell(y,0)+\ell(y,1)}{2}\}$ (since each unseen label appears with $1/2$ probability conditioned on the learner's sample). The total expected loss of any strategy is then at least $\ell_{\text{min-err}}/2$, which is bounded away from $0$. Setting $\varepsilon$ and $c$ sufficiently small then gives the desired result.
\end{proof}
\Cref{prop:infinite-lower} relies crucially on the fact that the adversary can erase a significant amount of information with a very small label perturbation. In the rest of this section, we'll discuss a technique for modifying our reduction that shows this is essentially \textit{the only barrier} between realizable and agnostic learning (at least for a broad class of loss functions). The key is to require a slightly stronger notion of learnability based upon \textit{discretization}.

\begin{definition}[Discretization]
We say $(\dist, X, H', \ell)$ is an $\varepsilon$-discretization of $(\dist, X, H, \ell)$ if the following three conditions hold:
\begin{enumerate}
    \item $H'$ is \textbf{probably bounded}. That is for all $n \in \mathbb{N}$, $\delta > 0$, and $D \in \dist$ there exists a bound $m(n,\delta) \in \mathbb{N}$ such that:
    \[
    \Pr_{S \sim D^n}[|Im(H'|_{S})| \leq m(n,\delta)] \geq 1-\delta.
    \]
    \item $H'$ \textbf{point-wise\footnote{We note that in most cases this condition can generally be weakened to the expected guarantee $\mathbb{E}_{x \sim D_X}[\ell(h'(x),h(x))] \leq \varepsilon$, but the stronger notion is needed for some applications such as adversarial robustness.} $\varepsilon$-covers} $H$ with respect to $\ell$. That is for all for all $h \in H$, there exists $h' \in H'$ satisfying:
    \[
    \forall x \in X: \ell(h'(x),h(x)) \leq \varepsilon.
    \]
    \item $H'$ is \textbf{always useful}. That is for all $h' \in H'$, there exists $h \in H$ such that:
    \[
        \forall x \in X: \ell(h'(x),h(x)) \leq \varepsilon.
    \]
\end{enumerate}
\end{definition}
\noindent Note that most realistic settings have reasonable discretizations (e.g.\ it is enough to have some Lipshitz-like condition and a weak tail-bound on the loss). We now define a basic notion of learnability based on discretization which essentially serves to rule out adversarial constructions in the vein of \Cref{prop:infinite-lower}.

\begin{definition}[Discretely-learnable]\label{def:discrete}
    We say $(\dist,X,H,\ell)$ is discretely-learnable with sample complexity $n(\varepsilon,\delta)$ if there is some constant $c_1>0$ such that for all $\varepsilon,\delta>0$ there exists an $\varepsilon$-discretization $H_{\varepsilon}$ which is $(c_1\varepsilon,\delta)$-PAC-learnable in at most $n(\varepsilon,\delta)$ samples. We call the learner proper if it outputs hypotheses in $H$.
\end{definition}
We'll prove that discrete and agnostic learnability are equivalent as long as the loss satisfies an approximate triangle inequality.
\begin{definition}[Approximate pseudometric]
    We call a loss function $\ell: Y \times Y \to \mathbb{R}_{\geq 0}$ a $c$-approximate pseudometric if for all triples $y_1,y_2,y_3 \in Y$:
        \[
        \ell(y_1,y_3) \leq c(\ell(y_1,y_2) + \ell(y_2,y_3)).
        \]
\end{definition}
Approximate pseudometrics are natural choices for loss functions in practice and capture a broad set of scenarios including finite-range losses and standard setups such as $\ell_p$-regression, and have seen some previous study in the literature \cite{crammer2008learning}. By modifying the first step of our reduction to take discretization into account and leveraging the approximate triangle inequality in the second, we prove that discrete learnability and $c$-agnostic learnability are equivalent under $c$-approximate pseudometrics.
\begin{theorem}\label{thm:approximate pseudometric}
Let $\ell: Y \times Y \to \Rplus$ be a bounded $c$-approximate pseudometric. Then the following are equivalent for all  $(\dist, X,H,\ell)$:
\begin{enumerate}
    \item $(\dist, X,H,\ell)$ is discretely-learnable.
    \item $(\dist, X,H,\ell)$ is $c$-agnostically learnable.
\end{enumerate}
\end{theorem}
\begin{proof}
The proof is similar to \Cref{thm:basic-reduction}. We first show the forward direction. Assume $(\dist, X, H, \ell)$ is discretely-learnable. Fix $\varepsilon'=\frac{\varepsilon}{4c^2c_1}$ (where $c_1$ is the constant from \Cref{def:discrete}), and let $H_{\varepsilon'}$ be a learnable $\varepsilon'$-discretization of $H$. We argue that running \textsc{LearningToCover} on $H_{\varepsilon'}$ gives the desired agnostic learner. Since $\ell$ is bounded, it is sufficient to prove that $C(S_U)$ contains a hypothesis $h'$ such that $\text{err}_{D,\ell}(h') \leq c \cdot OPT + \varepsilon/2$. Empirical risk minimization then works as in the finite case.

Let $h_{OPT} \in H$ be an optimal hypothesis. Since $H_{\varepsilon'}$ is a discretization of $H$, there exists $h^{\varepsilon'}_{OPT} \in H_{\varepsilon'}$ such that: 
\[
\forall x \in X: \ell(h_{OPT}(x),h^{\varepsilon'}_{OPT}(x)) < \varepsilon'.
\]
Further, by the guarantees of discrete learnability, with probability at least $1-\delta/2$ there exists $h' \in C(S_U)$ such that close to $h^{\varepsilon'}_{OPT}$ in the following sense:
\[
\underset{x \sim D_X}{\mathbb{E}}[\ell(h'(x),h^{\varepsilon'}_{OPT}(x))] \leq c_1\varepsilon'= \frac{\varepsilon}{4c^2}
\]
Plugging in the previous observation and applying our approximate triangle inequality, we get that $h'$ is close to $h_{OPT}$ in the following sense:
\begin{align*}
\underset{x \sim D_X}{\mathbb{E}}[\ell(h'(x),h_{OPT}(x))] &\leq c\left(\underset{x \sim D_X}{\mathbb{E}}[\ell(h'(x),h^{\varepsilon'}_{OPT}(x))] + \underset{x \sim D_X}{\mathbb{E}}[\ell(h^{\varepsilon'}_{OPT}(x),h_{OPT}(x))]\right) \\
&\leq \frac{\varepsilon}{2c}
\end{align*}
The final step is to transfer from the marginal $D_X$ to the full joint distribution of the adversary, which follows immediately from a similar application of the approximate triangle inequality. This is the only step that loses a factor in the OPT term:
\begin{align*}
\text{err}_{D,\ell}(h') &= \underset{(x,y) \sim D}{\mathbb{E}}[\ell(h'(x),y)]\\
&\leq c\left(\underset{(x,y) \sim D}{\mathbb{E}}[\ell(h'(x),h_{OPT}(x))] + \underset{(x,y) \sim D}{\mathbb{E}}[\ell(h_{OPT}(x),y)]\right)\\
&\leq c \cdot OPT + \varepsilon/2
\end{align*}
as desired.

We now prove the reverse direction, which is essentially immediate. Assume the existence of a $c$-agnostic learner for $(\dist, X,H,\ell)$. Given a discretization $H_\varepsilon$, we want to show $(\dist, X,H_\varepsilon,\ell)$ is learnable to within $c_1\varepsilon$ error for some $c_1>0$. This is achieved simply by running the agnostic learner on $(\dist, X,H_\varepsilon,\ell)$. Since $H_{\varepsilon}$ is ``always useful'', every $h \in H_{\varepsilon}$ is $\varepsilon$-close to some $h' \in H$ in the sense that:
\[
\forall x \in X: \ell(h'(x),h(x)) \leq \varepsilon.
\]
In particular, this means that for any choice of $h$ by the adversary there exists $h' \in H$ with low error:
\[
\text{err}_{D,\ell}(h') = \mathbb{E}[\ell(h'(x),h(x))] \leq \varepsilon.
\]
As a result, running the $c$-agnostic learner for $(\dist, X,H,\ell)$ returns a hypothesis of at most $(c+1)\varepsilon$ error with high probability.
\end{proof}
It is worth noting that bounded loss is not really necessary for \Cref{thm:approximate pseudometric}. More generally we can require that $(\dist,X,H,\ell)$ is ``finitely learnable'' in the sense that for all finite subsets $H' \subset H$, $(\dist,X,H',\ell)$ is agnostically learnable. When $\ell$ is bounded, this is true for any finite class by empirical risk minimization. When $\ell$ is unbounded, we can still get finite learnability if $\ell(h(\cdot), \cdot)$ satisfies some concentration bounds \cite{unboundedexp2021} (such classes are discretizable since one can simply truncate the loss). It is also possible to achieve improved convergence rates by using algorithms other than empirical risk minimization for the finite learner (for example a generalization of the median-of-means estimator \cite{daniel2016}), but our reduction will suffer an additional factor of $\varepsilon$ over applying such techniques directly due to the size of the non-uniform cover.

It is also worth noting that various modifications to the definition of loss (e.g. defining loss between hypotheses rather than on $Y$ directly) will continue to work with the above. Similarly, there are various cases when one can get better than $c \cdot OPT$ accuracy for a $c$-approximate pseudometric, generally by instead optimizing over some surrogate loss function. For instance, if a simple transformation of the loss gives a $c'$-approximate pseudometric for $c' < c$, then one can generally learn up to $c' \cdot OPT$.\footnote{This does require that $OPT$ is bounded.} As an example, note that while square loss $\ell_2(x,y)=(x-y)^2$ is a 2-approximate pseudometric, taking $\sqrt{\text{err}_{D,\ell_2}}$ gives a true metric between hypotheses. As a result, as long as $OPT$ is bounded, we can get truly agnostic learning by optimizing $\sqrt{\text{err}_{D,\ell_2}}$ instead. This strategy works for any polynomial loss, such as $\ell_p(x,y) = |x-y|^p$.

On the other hand, outside of these special cases, \Cref{thm:approximate pseudometric} is tight: there exist $c$-approximate pseudometric loss functions which cannot be $c'$-agnostically learned for any $c'<c$. The argument is similar to \Cref{prop:infinite-lower}, but requires a bit more care.
\begin{proposition}\label{prop:c-agn-lower}
There exists a discretely-learnable class over a $c$-approximate pseudometric that is not $c'$-agnostically learnable for any $c'<c$.
\end{proposition}
\begin{proof}
    The proof is similar to \Cref{prop:lowerbound-agnostic}. We consider the same instance space $X = \N$ and hypothesis class $H$:
    \begin{align*}
        H = \{h : h(x) = (0, \cdot) ~\forall ~x\in X\}.
    \end{align*} 
    The loss function $\ell: Y \times Y \to \{0,1,c\}$ is also the same, but extended to the larger domain $Y = \N^2$: 
    \begin{align*}
        \ell((b_1, r_1), (b_2,r_2)) = \begin{cases}
        0 & b_1 = b_2\\
        1 & b_1 \neq b_2 ~\text{and}~ r_1 = r_2\\
        c & ~\text{otherwise}.
        \end{cases}
    \end{align*} 
    As before, note that $(X, H, \ell)$ is trivially realizably learnable by always returning any $h \in H$, $\ell$ is a $c$-psuedometric by definition, and for any labeling $f: X \to Y$ there exists $h \in H$ such that for all distributions $D$: 
    \[
    OPT \leq err_{D,f}(h) \leq 1.
    \]
    We now show that the class $(X, H, \ell)$ is only $c$-agnostically learnable. Since $OPT \leq 1$, it suffices to show that for every $m \in \mathbb{N}$, large enough $n \in \mathbb{N}$, and randomized algorithm $\algo$ on $m$ samples, there exists a labeling $f: X \to Y$ and a marginal distribution $D_X$ such that:
        \begin{align}\label{eq:lower-bound2}
         \Ebb_{S \sim D^m_{X}} \Ebb_{x \sim D_{X}} [\ell(\algo(S,f(S))(x), f(x))] \geq \Big(1-\frac{1}{n}\Big)^3c.
        \end{align} 
    For $n \geq \frac{1}{1-(1-(c-c')/c)^{1/3}}$, applying Markov's inequality to \Cref{eq:lower-bound2} implies that $\algo$ has error at least $c'$ with constant probability.
    
    For simplicity, we now restrict our attention to the marginal $D_X$ which is uniform over the set $[k]$ for some $k\in\mathbb{N}$ to be fixed. By Yao's minimax principle, its enough to prove that there exists a \emph{distribution} $\mu$ \emph{over functions} $f : [k] \to [n]^2$ such that any \emph{deterministic} algorithm $\algo$ the following holds 
    \begin{align*}
     \Ebb_{f\sim \mu} \Ebb_{S \sim D^m_{X}}\Ebb_{x\sim D_X}[\ell(\algo(S,f(S))(x),f(x))] \geq \Big(1-\frac{1}{n}\Big)^3 c.
    \end{align*} 
    We now show that the above holds for $\mu$ being uniform over all functions from $[k]$ to $[n]^2$ when $k > \frac{2m}{\ln(n/(n-1))}$. Similar to \Cref{prop:lowerbound-agnostic}, we have that 
    \begin{align*}
        &\Ebb_{x\sim D_X}\Ebb_{f\sim \mu} \Ebb_{S \sim D^m_{X}}[\ell(\algo(S,f(S))(x), f(x))] \geq \Ebb_{x\sim D_X}\Big[\Pr_{S \sim D^m_{X}}[x \notin S] \cdot \Big(1-\frac{1}{n}\Big)^2 \cdot c\Big],
    \end{align*} 
    since no matter the assignment $\algo$ gives to $x \notin S$, it will be wrong on both coordinates with probability $(1 - 1/n)^2$ over the randomness of $\mu$. The result follows by noting that for every $x \in [k]$ 
    \begin{align*}
    \Pr_{S \sim D^m_{X \times f(X)}}[(x,\cdot) \notin S] = \Big(1-\frac{1}{k}\Big)^m \geq 1 - \frac{1}{n}
    \end{align*} 
    since $1-x \geq \exp\{-x/(1-x)\}$, and we have assumed $k > \frac{2m}{\ln(n/(n-1))}$ and $n > 1$. 
    Therefore, we get that 
    \begin{align*}
        &\Ebb_{x\sim D_X}\Ebb_{f\sim \mu} \Ebb_{S \sim D^m_{X}}[\ell(\algo(S,f(S))(x), f(x))]\geq \Big(1-\frac{1}{n}\Big)^3 c
    \end{align*} 
    which completes the proof.
\end{proof}
\subsection{Sub-sampling: Malicious Noise}\label{sec:subsample}
Now that we've seen how to handle practical problems like regression over infinite label spaces, we'll discuss a technique that helps handle data corruption and data-dependent assumptions: sub-sampling. The main idea is as follows. Say that the original unlabeled sample we draw is, in some sense, partially corrupted: perhaps an adversary has changed some fraction of examples (malicious noise), or some portion of the sample is un-realizable for a concept in the class (robust and partial learning). In either case, there generally exists a core subset of ``clean'' samples that we can use to recover the guarantees of \textsc{LearningToCover}. Since we cannot necessarily \textit{identify} these, the idea is to run \textsc{LearningToCover} over enough subsets of the unlabeled sample that we find a clean subsample with high probability. In this section we'll discuss the application of this technique in detail to Kearns and Li's \cite{kearns1993learning} well-studied malicious noise model. In Sections \ref{sec:robust} and \ref{sec:partial}, we discuss further applications of the method to adversarially robust and partial learning.

To start, let's recall the standard malicious noise model. In this variant of PAC learning, instead of having access to the standard sample oracle from the adversary's distribution $D$ over $X \times Y$, we have access to a \textit{malicious oracle} $\bigo_M(\cdot)$ which, with probability $\eta$, outputs an adversarially chosen pair $(x,y)$, and otherwise samples from $D$ as usual.
\begin{definition}[PAC-learning with Malicious Noise]
We call $(X, H, \dist, \ell)$ (agnostically) $(\varepsilon,\delta)$-learnable with malicious noise at rate $\eta=\eta(\varepsilon)$ if there exists an algorithm $\algo$ and function $m=m_{\text{mal}}(\varepsilon,\delta)$ such that for all $\varepsilon,\delta > 0,$ and distributions $D$ over $X \times Y$ satisfying:
\begin{enumerate}
    \item The marginal $D_X \in \dist$,
\end{enumerate}
$\algo$ outputs a good hypothesis with high probability over samples drawn from the malicious oracle of size $n(\varepsilon,\delta)$:
\[
\Pr_{S \sim \bigo_M(\cdot)^{m}}[\text{err}_{D,\ell}(\algo(S)) > OPT + \varepsilon] \leq \delta.
\]
where $OPT = \min_{h \in H}\{\text{err}_{D,\ell}(h)\}$.
\end{definition}
In other words, malicious noise essentially gives a worst-case formalization of the idea that an $\eta$-fraction of the learner's data is (adversarial) garbage. 

Let's now formalize the argument above: modifying \textsc{LearningToCover} to run over subsamples gives a sample-efficient algorithm for learning with malicious noise. For readability, we'll (somewhat informally) restate the algorithm with this change.

\begin{algorithm}[H]
\SetAlgoLined
\textbf{Input:} Realizable PAC-Learner $\algo$, Accuracy Parameter $\varepsilon < 1/2$, Noise Parameter $\eta < \frac{\varepsilon}{1+\varepsilon}$, Unlabeled Sample Oracle $\bigo_U$, Labeled Sample Oracle $\bigo_L$
\\
\textbf{Algorithm:}
\begin{enumerate}[leftmargin=*]
    \item Draw an unlabeled sample $S_U \sim \bigo_U$, and labeled sample $S_L \sim \bigo_L$.
    \item Run $\textsc{LearningToCover}$ over all $S \in (S_U)_{\eta'} \coloneqq \{ S \subseteq S_U: |S| = \lfloor\left(1-\eta'\right)|S_U|\rfloor \}$, where:
    \[
    \eta' = \frac{3\eta + \varepsilon/(1-\varepsilon)}{4}.
    \]
    \item \textbf{Return} the hypothesis in $C(S_U) \coloneqq \underset{S \in (S_U)_{\eta'}}{\bigcup}C(S)$ with lowest empirical error over $S_L$.
\end{enumerate}
 \caption{Malicious to Realizable Reduction}
 \label{alg:malicious}
\end{algorithm}

We now prove that \Cref{alg:malicious} gives an (agnostic) learner that is tolerant to malicious noise.
\begin{theorem}\label{thm:malicious}
    Let $(\dist, X, H)$ be realizably PAC-learnable with sample complexity $n(\varepsilon,\delta)$. Then for any $\eta < \frac{\varepsilon}{1+\varepsilon}$, \Cref{alg:malicious} is an agnostic learner for $(\dist, X, H)$ tolerant to $\eta$ malicious noise. Furthermore letting $\Delta = \frac{\varepsilon}{1+\varepsilon}-\eta$ and $\beta = \left(1+\frac{\eta}{\Delta}\right)\log\left(\frac{1}{\eta}\right)$, its sample complexity is at most:
    \[
    m_{\text{mal}}(\varepsilon,\delta) \leq O\left(\beta\frac{n(\frac{\Delta}{4}, \frac{\delta}{4})}{\Delta} +\frac{\log\left( \Pi_H\left(O\left(n(\frac{\Delta}{2}, \frac{\delta}{2}) + \frac{\eta^2\log(1/\delta)}{\Delta^2}\right)\right)\right) + \log(1/\delta)}{\Delta^2} +  \beta\frac{\eta^2\log(1/\delta)}{\Delta^3}\right)
    \]
    where we've assumed $\varepsilon < 1/2$ for simplicity.
\end{theorem}
\begin{proof}
To start, we'll review for completeness a fairly standard analysis of empirical risk minimization under malicious noise. Assume for the moment that the output of \textsc{LearningToCover}, $C(S_U)$, contains a hypothesis $h'$ satisfying $err(h') \leq OPT+\beta_1$. Say we draw $M$ labeled samples for the ERM step, and an $\eta'=\eta+\beta_2$ fraction are corrupted by the adversary. For large enough $M$, we can assume by a Chernoff and union bound that the empirical loss of every hypothesis returned by \textsc{LearningToCover} is at most some $\beta_3$ away from its true loss on the un-corrupted portion of $M$ (we will make all these assumptions formal in a moment). Given these facts, notice that the empirical loss of $h'$ on $M$ is at most:
\[
err_M(h') \leq (1-\eta')(OPT + \beta_1+\beta_3) + \eta'.
\]
On the other hand, the empirical error of any $h_\varepsilon \in H$ whose true error is greater than $OPT+\varepsilon$ is at least:
\[
err_M(h) > (1-\eta')(OPT + \varepsilon-\beta_3).
\]
To ensure that our ERM works, it is enough to show that for any such $h$, $err_M(h) > err_M(h')$. A simple calculation shows that this is satisfied as long as $\beta_1+2\beta_3 \leq \varepsilon$ and $\beta_1+\beta_2 + 2\beta_3 \leq \Delta$. Setting $\beta_1=\beta_2=\beta_3=\Delta/4$ gives the desired result. 

It is left to argue that our assumptions above hold with high probability. First, note that by a Chernoff bound, the probability that $\eta' \geq \eta+\Delta/4$ on a set of $M$ samples is at most $e^{-c\left(\frac{\Delta}{\eta}\right)^2M}$ for some $c>0$, so this occurs with high probability as long as $M \geq \Omega(\log(1/\delta)\eta^2/\Delta^2 )$. Similarly, the empirical error of every hypothesis in $C(S_U)$ on the remaining clean samples will be within $\Delta/4$ of its true error as long as $M \geq \Omega(\log(|C(S_U)|/\delta)\eta^2/\Delta^2)$.

It is then left to show that $C(S_U)$ contains a hypothesis of error at most $OPT+\Delta/4$. To show this, it is enough to ensure that we run \textsc{LearningToCover} over a clean subsample of size at least $n(\Delta/4, \delta/4)$ with high probability. If we draw $|S_U|=O\left(\frac{n(\Delta/4,\delta/4)}{1-\eta'}+\log(1/\delta)\frac{\eta^2}{\Delta^2}\right)$ unlabeled samples, a similar Chernoff bound to the above promises that at most an $\eta'$ fraction are corrupted with high probability, and therefore that at least $n(\Delta/4,\delta/4)$ remain un-corrupted. Running \textsc{LearningToCover} over all subsets of size $(1-\eta')|S_u|$ then gives the desired result. The sample complexity bound follows from choosing $M$ large enough to satisfy the above conditions along with the fact that $|C(S_U)| \leq {n \choose (1-\eta')n}\Pi_H(n)$.
\end{proof}
A few remarks are in order. First, the error tolerance of \Cref{thm:malicious} is tight. In their original introduction of malicious noise, Kearns and Li \cite{kearns1993learning} proved that for most non-trivial concept classes, no PAC-learner can be tolerate $\frac{\varepsilon}{1+\varepsilon}$ malicious noise. Second, we note that Kearns and Li also give a reduction procedure from a base learner which achieves better dependence on $\Delta$ above, however it allows the learner to query specifically for positive or negative examples which is a stronger setup than our model. Finally, we note in the case of finite VC, the above is off from known (randomized) bounds by a factor of roughly $\Delta^2$ \cite{cesa1999sample}. On the other hand, \Cref{thm:malicious} has the advantage of extending far beyond the VC setting, including to the general loss models discussed in \Cref{sec:infinite}. The proof remains mostly the same (with optimal error tolerance differing slightly), and is omitted.

Since our agnostic model restricts the adversary to choosing a distribution whose marginal lies in the original family,
\Cref{thm:malicious} provides the first insight on robustness against an adversary who can corrupt the \textit{underlying data} as well as the labels. One might wonder whether this result can be pushed further: is it possible to be robust against an adversary who can corrupt the marginal over $X$ in some stronger sense? Unfortunately, the answer is no: malicious noise is necessarily the most distributional corruption we can handle. Let's look at two basic lower bounds to see why. First, we'll consider an adversary who can \textit{remove} a portion of the learner's sample.
\begin{proposition}\label{prop:remove-bound}
For any $\delta > 0$, there exists a class $(\dist,X,H)$ which is realizably PAC-learnable, but not learnable under an adversary who can remove a $\delta$ fraction of the learner's sample.
\end{proposition}
\begin{proof}
This follows from a result of Dudley, Kulkarni, Richardson, and Zeitouni \cite{dudley1994metric} that there exists an unlearnable class $(\dist,X,H)$ such that for some $n(\varepsilon,\delta)$, $(D,X,H)$ is learnable in $n(\varepsilon,\delta)$ samples for every $D \in \dist$. The lower bound then follows simply from adding an extra unique identifying point $x_D$ to $X$ for every distribution $D$, and modifying each $D \in \dist$ to have $\Theta(\delta)$ support on $x_D$. This modified class is clearly learnable, since after drawing $O(1/\delta)$ samples, the learner will draw $x_D$ and identify the distribution $D$ with good probability. However, the class is not learnable under an adversary who removes points, since with high probability the adversary can completely remove any mention of $x_D$ from the learner's sample, reducing to the original unlearnable class $(\dist,X,H)$.
\end{proof}
An adversary who can \textit{add} samples is similarly powerful. Consider a model in which the adversary, after choosing a marginal distribution and labeling, sees the learner's sample ahead of time and may add an additional $\gamma$ fraction of correctly labeled elements.\footnote{For simplicity, we'll assume the adversary is restricted to picking a marginal distribution and labeling rather than any joint distribution over $X \times Y$.} While the realizable setting is largely unaffected by such an adversary (as any non-optimal hypothesis will still see some bad example), even near-trivial concept classes cannot be agnostically learned.
\begin{proposition}\label{prop:add-bound}
There exists a class $(X,H)$ which for any $\gamma > 0$ is realizably but not agnostically learnable under an adversary who can add a $\gamma$ fraction of correctly labeled points to the learner's sample.
\end{proposition}
\begin{proof}
Let $X=\{x,x_1,x_2\}$ and $H=\{h_1,h_2\}$ be any class such that $h_1(x)=h_2(x)$, but $h_1(x_i) \neq h_2(x_i)$ for $i=1,2$. In the realizable setting, note that a single labeled example on $x_1$ or $x_2$ exactly determines the hypothesis. As long as there is less than $1-\varepsilon$ mass on $x$, the learner will draw such a sample after $O(1/\varepsilon)$ samples with good probability. Further, if the mass on $x$ is at least $1-\varepsilon$, then $h_1$ and $h_2$ are both valid outputs. As a result, any ERM is a valid PAC-learner. Since adding correctly labeled examples can only help this learner, the class remains realizably learnable under an adversary who can add an arbitrary number of clean samples.

In the agnostic setting, consider an adversary who chooses a labeling $f$ such that $f(x)=h_1(x)=h_2(x)$, $f(x_1)=h_1(x_1)$, and $f(x_2)=h_2(x_2)$. The optimal hypothesis $h_{OPT}$ is then decided by the amount of mass on $x_1$ and $x_2$ in the marginal distribution. Namely if the adversary chooses a distribution $D$ over $\{x,x_1,x_2\}$, the optimal error is $\min\{D(x_1),D(x_2)\}$. The idea is then to note that for $\gamma' \leq \gamma/4$, the learner cannot distinguish between the two following distributions:
\[
D_1(x_1) = c_1\gamma', D_1(x_2) = (1-c_1)\gamma'
\]
and
\[
D_2(x_1) = (1-c_1)\gamma', D_2(x_2) = c_1\gamma'
\]
where $1/2>c_1>0$ is some small constant. Informally, if the two distributions are indistinguishable, any learner will always incur error of around $\frac{(1-c_1)\gamma'}{2}$, whereas OPT is $c_1\gamma'$ for both distributions.

Let's now give the formal argument. By Yao's Minimax Principle it is enough to prove there is a strategy over distributions such that any deterministic learner has high error. In particular, if we can prove that the expected error is at least $3\cdot OPT$, then $\Pr[\text{error} \geq 2\cdot OPT] \geq OPT$. Since OPT is just some constant $c_1\gamma'$ (dependent only on $\gamma$), this is sufficient to prove the result. Moving on, consider the strategy in which the adversary chooses the labeling described above, and chooses each marginal ($D_1$ or $D_2$) with probability $1/2$. We'll break our analysis into two cases dependent on the sample complexity of the learner. If the learner uses $O(1/\gamma')$ examples, then there is a constant probability of drawing a sample only consisting of the point $x$. Let $f'$ be the hypothesis returned by the deterministic learner on this sample. By construction, $f'$ must disagree with either $h_1$ or $h_2$ on $x_1$ or $x_2$. Assume $f'$ differs on $h_1(x_1)$ (the other cases will follow similarly). When the distribution is $D_2$, $f'$ has error at least $(1-c_1)\gamma'$. Since this occurs with constant probability \emph{independent} of the choice of $c_1$, choosing $c_1$ sufficiently small leads to an expected error of at least $3c_1\gamma'$ as desired.

On the other hand, when there are $n=\Omega(1/\gamma')$ samples, we claim that the adversary can force the following sample to occur with constant probability: $2\gamma'n$ instances of $x_1$ and $x_2$, and $(1+\gamma)n-4\gamma'n$ instances of $x$. This follows from the fact that for the appropriate choice of constant for $n$, a Chernoff bound gives that both $x_1$ and $x_2$ occur at most $2\gamma'n$ times with constant probability. Since the adversary is allowed to add $\gamma n \geq 4\gamma' n$ arbitrary examples, they can add instances of $x_1$, $x_2$, and $x$ until the above sample is achieved. The remainder of the argument is then the same as the previous case, as any learner response on this sample will incur similarly high expected error.
\end{proof}

It is also reasonable to consider distributional corruption in the semi-supervised setting, where the unlabeled and labeled data-sets might have different underlying distributions. We discuss this model in \Cref{sec:covariate}.

\subsection{Replacing ERM: Semi-Private Learning}\label{sec:privacy}
So far we have focused on property generalization for two forms of noise-tolerance---agnostic learning and learning with malicious noise. In this section, we'll show how to use \Cref{Intro:alg} to generalize a broader spectrum of finitely-satisfiable properties through replacing the ERM process with a generic finite learner with the desired property. Our prototypical example will be \textit{privacy}, which is well known to be finitely-satisfiable via McSherry and Talwar's \cite{mcsherry2007mechanism} \textit{exponential mechanism}. To start, we'll cover a few basic privacy definitions.

\begin{definition}[Differential Privacy]
    A learning algorithm $\algo$ is said to be $\alpha$-differentially private if for all neighboring inputs $S,S'$ which differ on a single example:
    \[
    \Pr[\algo(S) \in T] \leq e^\alpha \Pr[\algo(S') \in T],
    \]
    for all measurable events $T$ in the range of $\algo$.
\end{definition}
The exponential mechanism is one of the most widely used techniques in privacy. Informally, the algorithm allows for differentially private selection of a ``good'' choice from a finite set of objects (potential hypotheses in our case). More formally, let $s: (X\times Y)^* \times H \to \R$ be a ``score'' function, and define ``sensitivity'' $\Delta_s$ to be 
\[
\Delta_s = \max_{h \in H} \max_{S, S'} |s(S, h) - s(S', h)| 
\] 
where $S, S'$ are two neighbouring datasets. The exponential mechanism selects an item with a good score with high probability, while maintaining privacy.
\begin{definition}[Exponential Mechanism \cite{mcsherry2007mechanism}]
    The exponential mechanism $M_E$ on inputs $S, H, s$ with privacy parameter $\alpha$ selects and outputs $h\in H$ with probability \[
    \frac{\exp(\frac{\alpha s(S, h)}{2\Delta})}{\sum_{h' \in H} \exp(\frac{\alpha s(S, h')}{2\Delta})}.
    \]
\end{definition}
It is well known that the exponential mechanism leads to a private learner for finite hypothesis classes under bounded loss.
\begin{theorem}[Theorem 3.4 \cite{kasiviswanathan2011can}]\label{thm:exp-mechanism}
Let $(\dist,X,H,\ell)$ be a finite class with a bounded loss function. Then the sample complexity of $\alpha$-differentially private learning $(\dist,X,H,\ell)$ is at most:
\[
n_{pri}(\alpha,\varepsilon,\delta) \leq O\left(\log\left(\frac{|H|}{\delta}\right)\max\{\varepsilon^{-2},\varepsilon^{-1}\alpha^{-1}\}\right).
\]
\end{theorem}
We note that \cite[Theorem 3.4]{kasiviswanathan2011can} only considers classification loss, but the extension to bounded loss is immediate. Unfortunately, even with the power of the exponential mechanism, privacy is a very restrictive condition in the general PAC framework, since we're most often interested in \textit{infinite} hypothesis sets. Indeed even \textit{improper} private learning requires finiteness of a highly restrictive measure known as representation dimension \cite{beimel2013characterizing}, which can be infinite for classes of VC dimension $1$. As a result, the past decade has seen the introduction of a number of weaker, more practical definitions of privacy. In this section we'll focus on a model introduced in 2013 by Beimel, Nissim, and Stemmer \cite{beimel2013private} called \textit{semi-private learning}.
\begin{definition}[Semi-Private Learning]
    We call a class $(\dist,X,H,\ell)$ semi-private PAC-Learnable if there exists an algorithm $\algo$ and two functions $n_{\text{pub}}=n_{\text{pub}}(\varepsilon,\delta, \alpha)$ and $n_{\text{pri}}=n_{\text{pri}}(\varepsilon,\delta, \alpha)$ such that for all $\varepsilon,\delta > 0$ and distributions $D$ over $X \times Y$ whose marginal $D_X$ is in $\dist$, $\algo$ satisfies the following:
    \begin{enumerate}
        \item $\algo$ outputs a good hypothesis with high probability:
        \[
        \Pr_{S_U \sim D_X^{n_{\text{pub}}}, S_L \sim D^{n_{\text{pri}}}}[\text{err}_{D,\ell}(\algo(S_U,S_L)) > OPT + \varepsilon] \leq \delta.
        \]
        \item $\algo$ is semi-private. That is for all $S_U \in X^{n_{\text{pub}}}$:
        \[
        \algo(S_U, \cdot) \text{ is } \alpha\text{-differentially private}.
        \]
    \end{enumerate}
\end{definition}
In other words, semi-private learning offers a model for applications where \textit{labeled data} is sensitive, but some (perhaps opt-in) users might not care about their participation itself being released. Unlike standard private learning, distribution-free semi-private classification is known to be characterized by VC dimension, just like realizable PAC-learning \cite{beimel2013private}. The best sample complexity bounds are due to Alon, Bassily, and Moran (ABM) \cite{alon2019limits}, who use uniform convergence to build a uniform cover for $H$ from unlabeled data, and then apply the exponential mechanism to the resulting cover.

Due to their reliance on uniform convergence, ABM's techniques fail in the more general settings we consider. Further, their use of uniform covers results in sub-optimal public sample complexity even for distribution-free classification. We prove in \Cref{sec:cover} that these objects require asymptotically more samples than non-uniform covers (at least in the distribution-family model), and therefore cannot be used to achieve optimal semi-private learning. We circumvent both of these issues by appealing directly to a realizable learner to build a weaker non-uniform cover. For readability, we first restate the algorithm here.

\begin{algorithm}[H]
\SetAlgoLined
\textbf{Input:} Realizable PAC-Learner $\algo$, Unlabeled Sample Oracle $\bigo_U$, Labeled Sample Oracle $\bigo_L$
\\
\textbf{Algorithm:}
\begin{enumerate}[leftmargin=*]
    \item Draw an unlabeled sample $S_U \sim \bigo_U$, and labeled sample $S_L \sim \bigo_L$.
    \item Run \textsc{LearningToCover} over $S_U$ to get $C(S_U)$.
    \item \textbf{Return} the hypothesis in $C(S_U)$ given by applying the exponential mechanism with respect to $S_L$.
\end{enumerate}
\caption{Semi-Private to Realizable Reduction}
\label{alg:semi-private}
\end{algorithm}

We prove that \Cref{alg:semi-private} gives a semi-private agnostic learner in the distribution-family setting.
\begin{theorem}[PAC-learning implies Semi-Private Learning]\label{thm:private-PAC}
Let $\ell: Y \times Y \to \Rplus$ be a bounded $c$-approximate pseudometric. Then the following are equivalent for all triples $(\dist, X,H,\ell)$:
\begin{enumerate}
    \item $(\dist, X,H,\ell)$ is discretely-learnable
    \item $(\dist, X,H,\ell)$ is $c$-agnostically, semi-private learnable.
\end{enumerate}
\end{theorem}
\begin{proof}
The proof is essentially the same as \Cref{thm:approximate pseudometric}. The only difference in the argument is to replace the generic ERM learner over the output of \textsc{LearningToCover} with the exponential mechanism \cite{kasiviswanathan2011can}. 
\end{proof}

Let's now take a look at what \Cref{thm:private-PAC} implies about the special case of distribution-free classification.

\begin{corollary}\label{cor:private-class}
Let $(X,H)$ be a class of VC-dimension $d$ with sample complexity $n(\varepsilon,\delta)$. The sample complexity of (agnostic) semi-private learning $(X,H)$ is at most:
\[
n_{\text{pub}}(\varepsilon,\delta,\alpha) \leq n(\varepsilon/2,\delta/2)
\]
and
\[
n_{\text{pri}}(\varepsilon,\delta,\alpha) \leq O\left(\left(d\log\left(\frac{n(\varepsilon/2,\delta/2)}{d}\right) + \log\left(\frac{1}{\delta}\right)\right)\max\{\varepsilon^{-2},\varepsilon^{-1}\alpha^{-1}\}\right).
\]
Further, the sample complexity of improperly (agnostic) semi-private learning $(X,H)$ is
\[
n_{\text{pub}}(\varepsilon,\delta,\alpha) \leq O\left(\frac{d+\log(1/\delta)}{\varepsilon}\right)
\]
and
\[
n_{\text{pri}}(\varepsilon,\delta,\alpha) \leq O\left(\left(d\log\left(\frac{1}{\varepsilon}\right) + \log\left(\frac{1}{\delta}\right)\right)\max\{\varepsilon^{-2},\varepsilon^{-1}\alpha^{-1}\}\right).
\]
\end{corollary}
\begin{proof}
\textsc{LearningToCover} uses $n(\varepsilon/2,\delta/2)$ samples by definition. In the improper case, Hanneke showed that $n(\varepsilon/2,\delta/2) \leq O\left(\frac{d+\log(1/\delta)}{\varepsilon}\right).$ Since the class has VC dimension $d$, the size of the resulting cover is at most $(e\cdot n(\varepsilon/2,\delta/2)/d)^d$, and the private sample complexity bound then follows from \Cref{thm:exp-mechanism}.
\end{proof}
This improves over the recent upper bound of ABM who showed that
\[
n_{\text{pub}}(\varepsilon,\delta,\alpha) \leq O\left(\frac{d\log(1/\varepsilon)+\log(1/\delta)}{\varepsilon}\right).
\]
In fact, for constant $d$ and $\delta$, \Cref{cor:private-class} completely resolves the unlabeled sample complexity of semi-private learning, as ABM \cite{alon2019limits} prove a matching lower bound.
\begin{theorem}[Public Lower Bound \cite{alon2019limits}]\label{thm:abm-lower}
    Every class that is not privately learnable requires at least $\Omega(1/\varepsilon)$ unlabeled samples to learn in the semi-private model under classification error.
\end{theorem}
On the other hand, we note that the \textit{private} sample complexity remains off by a log factor from the best known lower bounds of Chaudhari and Hsu \cite{chaudhuri2011sample}.
\begin{theorem}[Private Lower Bound \cite{chaudhuri2011sample}]
There exist classes of VC dimension $O(1)$ which require at least:
\[
n_{\text{pri}}(\varepsilon,\delta,\alpha) \geq \Omega\left( \max\{\varepsilon^{-2},\varepsilon^{-1}\alpha^{-1}\}\right)
\]
private samples to learn.
\end{theorem}
While we have now resolved the public sample complexity of improper learning, it remains an interesting open problem for the proper regime where certain adversarial examples are known to require an extra $\log(1/\varepsilon)$ factor in the standard PAC sample complexity \cite{daniely2014optimal,44252}. We conjecture that \Cref{thm:private-PAC} should still be tight in this setting: namely that the unlabeled semi-private sample complexity should always be at least the realizable PAC sample complexity.

\begin{conjecture}
Let $(X,H)$ be a hypothesis class which is not privately learnable. Then the realizable sample complexity of $(X,H)$ lower bounds the unlabeled sample complexity of proper semi-private learning:
\[
n_{\text{pub}}(\varepsilon,1/2) \geq \Omega(n(\varepsilon,1/2)).
\]
\end{conjecture}

\subsection{Changing the Base Model: Covariate Shift}\label{sec:covariate}
One issue with semi-supervised models like semi-private learning is that, in practice, the distribution over unlabeled data probably won't match the labeled data exactly. In this section, we'll talk about a final modification to our reduction to tackle such scenarios and more generally to extend property generalization beyond the realizable PAC setting: replacing the base learner. In fact, we already saw this strategy used to a lesser extent in \Cref{sec:infinite}, where we replaced our standard realizable base learner with a discrete learner. Here we'll look at an application in which we assume our initial learner is robust to \emph{covariate shift} \cite{shimodaira2000improving}, meaning that even if the distribution underlying the data shifts between train and test time, the algorithm will continue to perform well. This stronger assumption will allow us to build semi-private learners that can handle corruption between the public and private databases. To start, let's formalize covariate shift in the distribution-family model. 
\begin{definition}[Covariate Shift]
    Let $(\dist, X, H, \ell)$ be any class, and for every $\varepsilon>0$ let $C_\varepsilon$ be a ``covariate-shift'' function that maps every $D \in \dist$ to some family of distributions over $X$. Given any distribution $D \in \dist$ and any $h \in H$, let the error of a potential labeling be given by its worst-case error over $C_\varepsilon(D)$, that is:
    \[
    \text{CS-err}_{D\times h,\ell,\varepsilon}(c) = \max_{D' \in C_\varepsilon(D)}\left\{\underset{{x \sim D'}}{\mathbb{E}}[\ell(c(x),h(x))]\right\}.
    \]
    We say that $(\dist, X, H, \ell)$ is realizably learnable under covariate shift $C=\{C_\varepsilon\}$ if there exists an algorithm $\algo$ and function $n=n(\varepsilon,\delta)$ such that for all $\varepsilon,\delta > 0$, $D \in \dist$, and $h \in H$:
    \[
    \Pr_{S \sim D^n}\left[\text{CS-err}_{D\times h,\ell,\varepsilon}(\algo(S,h(S))) > \varepsilon \right] \leq \delta.
    \]
    We call such a learner robust to covariate shift.
\end{definition}
We emphasize that in the above definition, the covariate shift family \emph{scales with the error parameter $\varepsilon$}. This is a bit different than Shimodaira's original definition \cite{shimodaira2000improving}, but is a natural choice in our context since we consider algorithms which only use access to the original source distribution (sometimes called ``conservative domain adaptation'' \cite{david2010impossibility}). In this setting, we'd expect that as we demand higher accuracy, the amount of covariate shift we can tolerate will decrease. Indeed in the agnostic model, it's clear this scaling is necessary by a similar argument to \Cref{prop:add-bound}.

The key observation to apply learning under covariate shift in our reduction is simply to notice that the non-uniform cover output by \textsc{LearningtoCover} must contain a hypothesis close to optimal under any shifted distribution in $C_\varepsilon(D)$. This can then be used to analyze any semi-supervised model where the marginal of the labeled distribution may be corrupted from $D$ to any distribution in $C_\varepsilon(D)$. In this section, we'll again focus on the setting of semi-private learning. First, let's formalize what it means to be semi-private learnable under covariate shift.
\begin{definition}[Semi-Private Learning under Covariate Shift]
    We call a class $(\dist,X,H,\ell)$ semi-private (agnostically) PAC-Learnable under covariate shift $C=\{C_\varepsilon\}$ if there exists an algorithm $\algo$ and two functions $n_{\text{pub}}=n_{\text{pub}}(\varepsilon,\delta, \alpha)$ and $n_{\text{pri}}=n_{\text{pri}}(\varepsilon,\delta, \alpha)$ such that for all $\varepsilon,\delta > 0$ and distributions $D_X \in \dist$ and $D'$ over $X \times Y$ whose marginal $D'_X \in C_\varepsilon(D_X)$, $\algo$ satisfies the following:
    \begin{enumerate}
        \item $\algo$ outputs a good hypothesis over $D'$ with high probability:
        \[
        \Pr_{S_U \sim D_X^{n_{\text{pub}}}, S_p \sim D'^{n_{\text{pri}}}}[\text{err}_{D',\ell}(\algo(S_U,S_p)) > OPT' + \varepsilon] \leq \delta,
        \]
        where $OPT'=\min_{h \in H}[\text{err}_{D',\ell}(h)]$ is the minimum error over the \textbf{shifted} distribution.
        \item $\algo$ is semi-private. That is for all $S_U \in X^{n_{\text{pub}}}$:
        \[
        \algo(S_U, \cdot) \text{ is } \alpha\text{-differentially private}
        \]
        \end{enumerate}
\end{definition}
In other words, we'd like to recover a near-optimal hypothesis even when the marginal distribution over private data is shifted from the public data. This is a realistic scenario in practice, since the distribution of ``opt-in'' users is likely different from the marginal over the total population. We'll show that this issue is solvable in the semi-private setting as long as the analogous issue in the non-private setting (distribution shift between train and test time) can be resolved.
\begin{theorem}\label{thm:covariate-shift}
    Let $(\dist,X,H,\ell)$ be any class of a finite label space with loss function satisfying the identity of indiscernibles which is realizably learnable under covariate shift $C=\{C_\varepsilon\}$. Then $(\dist,X,H,\ell)$ is semi-private, agnostically learnable under covariate shift $C$ with sample complexity:
    \[
    m(\varepsilon,\delta) \leq n(\eta_\ell\varepsilon,\delta/2) + O\left(\frac{\log\left(\frac{\Pi_H(n(\eta_\ell\varepsilon,\delta/2))}{\delta}\right)}{\varepsilon^2}\right)
    \]
    where we recall $\eta_\ell \geq \Omega\left(\frac{\min_{a \neq b}(\ell(a,b))}{\max_{a \neq b}(\ell(a,b))}\right)$ is a constant depending only on $\ell$.
\end{theorem}
\begin{proof}
The proof is essentially the same as \Cref{thm:basic-reduction}. The only difference is to note that for all marginal distributions $D_X \in \mathscr{D}$ and choices of shift $D'_X \in C_\varepsilon(D_X)$, the output of \textsc{LearningToCover} has some $h'$ satisfying
\[
\underset{x \sim D'_X}{\mathbb{E}}[\ell(h'(x),h_{OPT'}(x))] \leq \eta_\ell \varepsilon, 
\]
where $h_{OPT'}$ is some optimal hypothesis over the shifted distribution $D'$. As in the standard analysis, this is guaranteed by including the output of the realizable learner on $h_{OPT'}$ (robustness to covariate shift promises this output is close under $D'_X$ with high probability). The remainder of the argument is exactly as in \Cref{thm:approximate pseudometric}, with the exception of working over the shifted distribution $D'_X$ instead of $D_X$.
\end{proof}
We note that this result can also be extended to the more general loss functions discussed in \Cref{sec:infinite} without too much difficulty with the appropriate definition of discrete learnability under covariate shift.

Since \Cref{thm:covariate-shift} is a bit abstract, let's take a look at one concrete application. Given a class $(X,H)$, the class-dependent total variation distance is a metric on distributions measuring the worst case distance across elements of $H \Delta H \coloneqq \{h \Delta h' : h,h' \in H\}$:
\[
TV_{H \Delta H}(D,D') \coloneqq \max_{h \in H \Delta H}\{\left|D(h) - D'(h)\right| \}.
\]
It is not hard to see that any realizable learner is robust to $O(\varepsilon)$ covariate shift in $TV_{H \Delta H}$ distance. We can then apply \Cref{thm:covariate-shift} to build a robust semi-private learner.
\begin{corollary}
Let $(X,H)$ be a class of VC-dimension $d$, and for every $\varepsilon>0$ and distribution $D$ over $X$, define a coviarate shift function:
\[
C_\varepsilon(D) \coloneqq \left\{ D' : TV_{H \Delta H}(D,D') \leq \varepsilon/2\right\}.
\]
Then $(X,H)$ is semi-private learnable under covariate shift $C=\{C_\varepsilon\}$ in only
\[
n_{\text{pub}}(\varepsilon,\delta,\alpha) \leq O\left(\frac{d+\log(1/\delta)}{\varepsilon}\right)
\]
unlabeled samples and
\[
n_{\text{pri}}(\varepsilon,\delta,\alpha) \leq O\left(\left(d\log\left(\frac{1}{\varepsilon}\right) + \log\left(\frac{1}{\delta}\right)\right)\max\{\varepsilon^{-2},\varepsilon^{-1}\alpha^{-1}\}\right)
\]
labeled samples.
\end{corollary}

Finally, we note again that our original learner in these results is robust to covariate shift despite having no access to samples from the new distribution. Unfortunately, this model does come with fairly strong lower bounds regarding the type of covariate shifts to which it is possible to be robust \cite{david2010impossibility}. One solution to this problem is to consider a relaxed variant called (non-conservative) domain adaptation, where the learner additionally has access to a small number of unlabeled samples from the test-time distribution. It is certainly possible to define an analog in the semi-private setting, but naively the use of unlabeled data from the private distribution breaks our reduction since privacy won't be preserved. We leave as an open question whether any sort of PAC-learner in the \textit{non}-conservative model could imply semi-private learners with stronger robustness to covariate shift. Some progress has been made in this direction recently by Bassily, Moran, and Nandi \cite{bassily2020learning} for distribution-free classification of halfspaces.

\section{Doubly Bounded Loss}\label{sec:doubly-bounded}
In this section we discuss a natural generalization of loss functions over finite label classes we call doubly bounded loss: for all distinct $y,y' \in Y$, we require $\ell(y,y') \in [a,b]$ for some $b \geq a>0$. This is trivially satisfied by any loss function on a finite label class that satisfies the identity of indiscernibles.
\begin{definition}[Doubly Bounded Loss]
    We say $\ell: Y \to Y$ is $(a,b)$-bounded if there exist $b \geq a>0$ for any distinct $y,y' \in Y$:
    \[
    \ell(y,y') \in [a,b].
    \]
\end{definition}
As discussed in \Cref{sec:infinite}, since we now allow $Y$ to be infinite, we need to work with discrete learnability instead of realizable learnability. We can use a slight modification to the discretization technique in \Cref{thm:approximate pseudometric} to prove the equivalence of discrete and agnostic learnability for doubly-bounded loss functions. Note that this is stronger than our guarantee for $c$-approximate pseudometrics, which only gives $c$-agnostic learnability.

\begin{theorem}\label{thm:bounded}
Let $\ell: Y \times Y \to \Rplus$ be an $(a,b)$-bounded loss function. Then for any class $(\dist, X, H, \ell)$ the following are equivalent:
\begin{enumerate}
    \item $(X, H, \dist, \ell)$ is (properly) discretely-learnable
    \item $(X, H, \dist, \ell)$ is (properly) agnostically learnable.
\end{enumerate}
\end{theorem}
\begin{proof}
The proof that agnostic learnability implies discrete learnability is the same as in \Cref{thm:approximate pseudometric}, so we focus only the forward direction. Assume $(\dist, X, H, \ell)$ is discretely-learnable. Fix $\varepsilon'=\frac{a\varepsilon}{4b}$, and let $H_{\varepsilon'}$ be a learnable $\varepsilon'$-discretization of $H$. We argue that running \textsc{LearningToCover} on $H_{\varepsilon'}$ (using the promised discrete learner) gives the desired agnostic learner. As before, it is sufficient to prove that $C(S_U)$ contains a hypothesis $h'$ such that $\text{err}_{D,\ell}(h') \leq OPT + \varepsilon/2$. Since $\ell$ is upper bounded, standard empirical risk minimization arguments then give the desired result.

To see why $C(S_U)$ has this property, recall from \Cref{lemma:cover} that for any $h \in H_{\varepsilon'}$, with probability at least $1-\delta/2$ there exists $h' \in C(S_U)$ that is $\varepsilon'$-close to $h$ in the following sense:
\[
\mathbb{E}_{x \sim D_X}[\ell(h'(x),h(x))] \leq \frac{a\varepsilon}{4b}.
\]
Because $\ell$ is $a$-lower bounded, this implies that $h'$ must be close to $h$ in classification error:
\[
\Pr_{x \sim D_X}[h(x) \neq h'(x)] \leq \frac{\varepsilon}{4b},
\]
and since the loss is bounded by $b$, the risk of $h'$ cannot be much more than of $h$:
\[
\text{err}_{D,\ell}(h') \leq \text{err}_{D,\ell}(h) + \varepsilon/4.
\]
Let $h_{OPT} \in H$ be an optimal hypothesis. Since $H_{\varepsilon'}$ $\varepsilon'$-covers $H$, there exists $h^{\varepsilon'}_{OPT} \in H_{\varepsilon'}$ such that:
\[
\text{err}_{D,\ell}(h^{\varepsilon'}_{OPT}) - \text{err}_{D,\ell}(h_{OPT}) \leq \varepsilon'.
\]
Let $h'^{\varepsilon'}_{OPT} \in H_{\varepsilon'}$ denote the output of the base learner $\algo$ on the labeling given by $h^{\varepsilon'}_{OPT}$. Then by the above, we have that:
\begin{align*}
    \text{err}_{D,\ell}(h'^{\varepsilon'}_{OPT}) &= \underset{(x,y) \sim D}{\mathbb{E}}[\ell(h'^{\varepsilon'}_{OPT}(x),y)]\\
    &\leq \underset{(x,y) \sim D}{\mathbb{E}}[\ell(h^{\varepsilon'}_{OPT}(x),y)] + \varepsilon/4\\
    &\leq \underset{(x,y) \sim D}{\mathbb{E}}[\ell(h^{\varepsilon'}_{OPT}(x),y) - \ell(h_{OPT}(x),y) + \ell(h_{OPT}(x),y)] + \varepsilon/4\\
    &= \text{err}_{D,\ell}(h^{\varepsilon'}_{OPT}) - \text{err}_{D,\ell}(h_{OPT}) + OPT + \varepsilon/4\\
    &\leq OPT + \varepsilon/2
\end{align*}
as desired.
\end{proof}
It is worth noting that the upper bound on the loss can be removed if the adversary is restricted to choosing a marginal over $Y$ which is weakly concentrated.

\section{Robust Learning}\label{sec:robust}
Robust learning is an extension to the PAC setting that models an adversary with the power to \textit{perturb} examples at test time. In practice, this corresponds to the fact that we'd like our predictors to be stable to small amounts of adversarial noise---this could range anywhere from a sticker on a stop-sign tricking a self-driving car, to completely imperceptible perturbations that totally fool standard classifiers. The latter was famously demonstrated by Athalye, Engstrom, Ilyas, and Kwok \cite{athalye2018synthesizing}, who showed how to generate such perturbations and provided the classic example of tricking a standard ImageNet classifier into thinking a turtle was a gun. Their seminal work caused an explosion of both practical and theoretical research in the area.\footnote{Their work has over 900 citations despite being only four years old.}

Formally, adversarial robustness is modeled simply by changing the error function to be the maximum error over some pre-defined set of neighboring perturbations.
\begin{definition}[Robust Loss]
    Let $X$ be an instance space and $\mathcal{U}: X \to P(X)$ a ``perturbation function'' mapping elements to a set of possible perturbations. Given a loss function $\ell: Y \times Y \to \Rplus$, the robust loss of a concept $h: X \to Y$ with respect to a distribution $D$ over $X \times Y$ is:
    \[
\text{R-err}_{\mathcal{U},D}(h) = \underset{(x,y) \sim D}{\mathbb{E}}\left[\max_{x' \in \mathcal{U}(x)} \ell(h(x'),y)\right].
\]
\end{definition}
In other words, a hypothesis with low robust loss performs well even against an adversary who can perturb $x$ to any ``nearby point'' (i.e.\ any $x' \in \mathcal{U}(x)$). Standard realizable and agnostic Robust PAC-learning are then simply defined by replacing the standard error function with the robust error function. Robust learning in the distribution-family model does require one extra twist: we need to make sure that each hypothesis in the class actually has a corresponding distribution over which it is realizable. To this end, we introduce a basic notion of closure for distribution families.
\begin{definition}[Robust Closure]
Let $\dist$ be a set of distributions over an instance space $X$ and $H$ a concept class. Given any concept $h$, let $X_{h}$ denote the set of points in $X$ on which $h$ has $0$ robust loss with respect to itself, that is:
\[
X_h \coloneqq \{ x \in \text{X} : \forall x' \in \mathcal{U}(x), \ell(h(x'),h(x))=0 \}.
\]
For notational simplicity, let $D|_h$ denote the restriction $D|_{X_h}$. The robust closure of $\dist$ under $H$ is:
\[
\mathscr{D}_H \coloneqq \dist \cup \bigcup_{D \in \mathscr{D}, h \in H} D|_h.
\]
\end{definition}
In the robust distribution-family model, it only really makes sense to define realizable learnability over the robust closure of $\mathscr{D}$, since otherwise there may be hypotheses in the class that are not realizable with respect to any distribution in $\dist$ and cannot be chosen by the adversary at all. With this in mind, let's formalize this model.
\begin{definition}[(Realizable) Distribution-Family Robust PAC Learning]
A class $(\dist, X, H, \ell)$ is Robustly PAC-learnable in the realizable setting with respect to perturbation function $\mathcal{U}$ if there exists an algorithm $\algo$ and function $n(\varepsilon,\delta)$ such that for all $\varepsilon,\delta > 0$ and distributions $D$ over $X\times Y$ satisfying:
\begin{enumerate}
    \item The marginal $D_X \in \dist_H$,
    \item $\min_{h \in H} \text{R-err}_{\mathcal{U},D}(h) = 0$,
\end{enumerate}
then:
\[
\Pr_{S \sim D^{n(\varepsilon,\delta)}}[\text{R-err}_{\mathcal{U},D}(\algo(S)) > \varepsilon] \leq \delta.
\]
\end{definition}
Agnostic learnability is defined similarly, but since the adversary is unrestricted, there is no longer any need to take the robust closure.
\begin{definition}[(Agnostic) Distribution-Family Robust PAC Learning]
A class $(\dist, X, H, \ell)$ is Robustly PAC-learnable in the agnostic setting with respect to perturbation function $\mathcal{U}$ if there exists an algorithm $\algo$ and function $n(\varepsilon,\delta)$ such that for all $\varepsilon,\delta > 0$ and distributions $D$ over $X\times Y$ satisfying $D_X \in \dist$, then:
\[
\Pr_{S \sim D^{n(\varepsilon,\delta)}}[\text{R-err}_{\mathcal{U},D}(\algo(S)) > OPT + \varepsilon] \leq \delta,
\]
where $OPT = \min_{h \in H}\{\text{R-err}_{\mathcal{U},D}(h)\}$.
\end{definition}
We note that different works consider different models of access to the perturbation set $\mathcal{U}$ as well (e.g.\ assuming $\mathcal{U}$ is known to the learner \cite{montasser2019vc}, or has some type of oracle access \cite{montasser2020reducing,montasser2021adversarially}). Our reduction requires fairly weak access to $\mathcal{U}$---it is enough to be able to estimate the empirical robust loss of a hypothesis $h$ over any finite sample $S \subset X$. With this in mind, let's now prove realizable and agnostic robust learning are equivalent in the distribution-family model. We'll focus on the special case of (multi-class) classification, and start by re-stating our modified algorithm for simplicity of presentation.

\begin{algorithm}[H]
\SetAlgoLined
\textbf{Input:} Realizable Robust PAC-Learner $\algo$
\\
\textbf{Algorithm:}
\begin{enumerate}[leftmargin=*]
    \item Draw an unlabeled sample $S_U$, and labeled sample $S_L$.
    \item Run $\algo$ over all possible subsets and labelings of $S_U$ to get:
    \[
    C(S) \coloneqq \{ \algo(S,h(S))~|~ S \subseteq S_U, h \in H|_S\}.
    \]
    \item \textbf{Return} the hypothesis in $C(S)$ with lowest empirical robust error over $S_L$.
\end{enumerate}
 \caption{Agnostic to Realizable Reduction Robust Setting}
 \label{alg:robust}
\end{algorithm}

\begin{theorem}\label{thm:robust}
    If $(\dist,X,H)$ is robustly PAC-learnable in the realizable setting with sample complexity $n(\varepsilon,\delta)$, then \Cref{alg:robust} robustly learns $(\dist,X,H)$ in the agnostic setting in:
\[
m_U(\varepsilon,\delta) \leq O\left(\max_{\mu \in [0,1-\varepsilon]}\left\{\frac{n(\varepsilon/(2(1-\mu)),\delta/3)}{1-\mu} \right\}\right)
\]
unlabeled samples, and
\[
m_L(\varepsilon,\delta) \leq O\left(\frac{m_U(\varepsilon,\delta) + \log(1/\delta)}{\varepsilon^2}\right)
\]
labeled samples.
\end{theorem}
\begin{proof}
The proof is similar to \Cref{thm:basic-reduction}, but like malicious noise (\Cref{thm:malicious}), requires the use of subsampling. The key issue is that for a distribution $D$ over $X \times Y$, the optimal hypothesis $h_{OPT}$ may not be realizable with respect to $D_X$, that is we may have:
\[
\text{R-err}_{U,D_X \times h_{OPT}}(h_{OPT}) > 0.
\]
As a result, our realizable learner (and therefore \textsc{LearningToCover}) has no guarantees over this distribution.
On the other hand, our realizable learner does have good guarantees over the \emph{restricted marginal} $D_{X}|_{h_{OPT}}$. We can then fix the above issue by running \textsc{LearningToCover} of all \textit{subsets} of $S$ including its restriction to $X_{h_{OPT}}$. We will see that this essentially simulates running the realizable learner on the realizable restriction $D_X|_{h_{OPT}}$ and recovers the desired guarantees.

Let's take a look at this more formally. As in our previous arguments it is enough to prove that $C(S)$ contains a hypothesis $h'$ with robust error at most $OPT + \varepsilon/2$, since a standard Chernoff bound tells us that $O(\log(|C(S)|/\delta)/\varepsilon^2)$ labeled examples are enough to estimate the robust loss of every hypothesis in $C(S)$ with high probability. We note that this is the only step in our reduction which requires access to the perturbation set $\mathcal{U}$. 

It is left to show that $C(S)$ satisfies this property. Let $D|_{h_{OPT}}$ denote the restriction of $D$ to $X_{h_{OPT}}$, the points in $X$ on which $h_{OPT}$ has $0$ robust loss with respect to itself. Let $\bar{D}|_{h_{OPT}}$ be the restriction to the complement, that is $X \setminus X_{h_{OPT}}$. The idea is to decompose our analysis into two separate parts over $D|_{h_{OPT}}$ and $\bar{D}|_{h_{OPT}}$. With this in mind, let $\mu^*$ denote the mass of $D_X$ on $X_{h_{OPT}}$, and let $OPT'$ denote the robust error of $h_{OPT}$ over $D|_{h_{OPT}}$. Since we are restricting our attention to classification error, notice that we can decompose $OPT$ as:
\begin{align*}
    OPT &= \text{R-err}_{\mathcal{U},D}(h_{OPT})\\
    &= \underset{(x,y) \sim D}{\Pr}\left[\exists x' \in \mathcal{U}(x): h_{OPT}(x') \neq y\right]\\
    &= \mu^*\underset{(x,y) \sim \bar{D}|_{h_{OPT}}}{\Pr}\left[\exists x' \in \mathcal{U}(x): h_{OPT}(x') \neq y\right]\\ 
    &+ (1-\mu^*)\underset{(x,y) \sim D|_{h_{OPT}}}{\Pr}\left[\exists x' \in \mathcal{U}(x): h_{OPT}(x') \neq y\right]\\
    &= \mu^* + (1-\mu^*)OPT',
\end{align*}
where the last step follows from noting that by definition for all $x$ in the support of $\bar{D}|_{h_{OPT}}$, $h_{OPT}$ is not constant on $\mathcal{U}(x)$. To get a function within $\varepsilon/2$ robust loss of $OPT$, we claim it is sufficient to prove $C(S)$ contains some $h$ within robust error $\varepsilon/(2(1-\mu^*))$ of $h_{OPT}$ over $D|_{h_{OPT}}$, that is some $h$ satisfying:
\begin{equation}\label{eq:robust}
\Pr_{x \sim D_X|_{h_{OPT}}}[\exists x' \in \mathcal{U}(x): h(x') \neq h_{OPT}(x')] \leq \varepsilon/(2(1-\mu^*)).
\end{equation}
This follows from a similar analysis to the above. Namely, letting $R(h,x,y)$ denote the event
\[
R(h,x,y) \coloneqq \exists x' \in \mathcal{U}(x): h(x') \neq y
\]
for notational simplicity, we can break down $\text{R-err}_{\mathcal{U},D}(h)$ as:
\begin{align*}
\text{R-err}_{\mathcal{U},D}(h) &= \mu^*\underset{(x,y) \sim \bar{D}|_{h_{OPT}}}{\Pr}\left[R(h,x,y)\right] + (1-\mu^*)\underset{(x,y) \sim D|_{h_{OPT}}}{\Pr}\left[R(h,x,y)\right]\\
&\leq \mu^* + (1-\mu^*)\underset{(x,y) \sim D|_{h_{OPT}}}{\Pr}\left[R(h,x,y)\right]\\
&\leq \mu^* + (1-\mu^*)\underset{(x,y) \sim D|_{h_{OPT}}}{\Pr}\left[R(h_{OPT},x,y)\right]\\ 
&+ (1-\mu^*)\underset{(x,y) \sim D|_{h_{OPT}}}{\Pr}\left[\exists x' \in \mathcal{U}(x): h(x') \neq h_{OPT}(x')\right] \\
&\leq \mu^* + (1-\mu^*)\left(OPT' + \frac{\varepsilon}{2(1-\mu^*)}  \right)\\
&= OPT + \varepsilon/2.
\end{align*}

It remains to prove that $C(S)$ contains a hypothesis satisfying \Cref{eq:robust} with high probability. To see this, note that by definition of realizable robust learning, on a labeled sample $(S,h_{OPT}(S)) \sim D|_{h_{OPT}} \times h_{OPT}$ of size $n(\varepsilon/(2(1-\mu^*)),\delta/3)$ our learner will output $h$ satisfying:
\[
\Pr_{x \sim D_X|_{h_{OPT}}}[\exists x' \in \mathcal{U}(x): h(x') \neq h_{OPT}(x)] \leq \varepsilon/(2(1-\mu^*))
\]
with probability at least $1-\delta/3$. To get \Cref{eq:robust}, it is then enough to note that $h_{OPT}$ is constant on $\mathcal{U}(x)$ for all $x$ in the support of $D_X|_{h_{OPT}}$ by definition.

The idea is now to draw a large enough unlabeled sample such that with probability at least $1-\delta/3$, the restriction to $X_{h_{OPT}}$ is at least this size. By a Chernoff bound, it is enough to draw $c_1\frac{n(\varepsilon/(1-\mu^*),\delta/3)}{1-\mu^*}$ points to achieve this for some large enough constant $c_1>0$.\footnote{We've assumed for simplicity that $n(\varepsilon,\delta) \geq \Omega(\log(1/\delta))$. This assumption can be removed by including an extra additive factor of $\log(1/\delta)$.} Since we do not know $\mu^*$, we'll need to draw $c_1 \max_{\mu \in [0,1-\varepsilon]}\left\{\frac{n(\varepsilon/(1-\mu),\delta/3)}{1-\mu} \right\}$ points to ensure this property holds (if $\mu^* \geq 1-\varepsilon$, note that any hypothesis gives a valid solution). By a union bound we have that this overall process succeeds with probability at least $1-2\delta/3$, and outputting the hypothesis in $C(S)$ with the lowest empirical robust risk then succeeds with probability $1-\delta$ as desired.
\end{proof}
\Cref{thm:robust} can be extended to many of the generic property generalization results in the main body, including approximate pseudometric loss, malicious noise, and semi-private learning, though the exact parameters may be somewhat weaker (e.g.\ learning over non-binary loss may incur additional factors and lead to $c$-agnostic rather than truly agnostic learning). This comes at the cost of an extra factor of $\varepsilon^{-1}$ over reductions in the finite VC setting \cite{montasser2019vc}.

\section{Partial PAC-Learning}\label{sec:partial}
Partial PAC-learning is an extension of the standard PAC model to functions that are only defined on a certain portion of the input. Originally introduced by Long \cite{long2001agnostic} and recently developed in greater depth by Alon, Hanneke, Holzman, and Moran (AHHM) \cite{alon2021theory}, this model allows for the theoretical formalization of popular data-dependent assumptions such as \emph{margin} that have no known analog in the PAC model. Combined with the distribution-family framework, this captures a significant portion of learning assumptions studied in both theory and practice (e.g.\ learning halfspaces with margin and distributional tail bounds). Let's formalize this model, starting with partial functions.

\begin{definition}[Partial Function]
Let $X$ be an instance space and $Y$ a label space. A partial function is a labeling $f: X \to Y \cup \{*\}$, where elements labeled ``$*$'' are thought as of undefined. The support of $f$, denoted $\text{supp}(f)$, is the set of elements $x \in X$ s.t.\ $f(x) \neq *$.
\end{definition}

Standard Partial PAC-learning is defined much like the standard model with the simple modification that ``$*$'' labels are always considered to be incorrect. As a result, in the realizable case, when the adversary selects a particular partial function $f$, their marginal distribution over the instance space $X$ must be restricted to lying on $\text{supp}(f)$. This makes formalizing data-dependent assumptions easy. If one wanted to consider halfspaces with margin $\gamma$ for instance, one simply labels every point within $\gamma$ of the decision boundary as ``$*$.'' Interestingly, much like the distribution-family setting, Partial-PAC learning falls outside both the uniform convergence and the sample compression paradigm \cite{ailon2018approximate}. AHHM also show a dramatic failure of empirical risk minimization: not only does naively applying an ERM to the partial class fail, it will also fail on any total extension of the class. Despite the lack of these standard tools, both Long and AHHM were able to show that distribution-free classification of partial classes is still controlled by VC dimension, and as a result that the equivalence of realizable and agnostic learnability extends to this setting. In this section, we'll discuss how a variant of our reduction shows that this result extends to the distribution-family model, extended loss function, and to properties beyond agnostic learning.

In the distribution-family model, formalizing realizable learnability requires some slight changes from the standard model, since we need to make sure our hypotheses are actually realizable over some distribution in the family (this is automatic in the distribution-free setting). To this end, we introduce a basic notion of closure for distribution families.
\begin{definition}[Partial Closure]
Let $\dist$ be a set of distributions over an instance space $X$ and $H$ a concept class. Given any concept $h$, and distribution $D$ over $X$, 
let $D|_h$ denote the restriction $D|_{\text{supp}(f)}$. The partial closure of $\dist$ under $H$ is:
\[
\mathscr{D}_H \coloneqq \dist \cup \bigcup_{D \in \mathscr{D}, h \in H} D|_h.
\]
\end{definition}
In the realizable model it makes more sense to work with the closure of $\dist$ than $\dist$ itself, since otherwise the class $H$ may contain hypotheses that cannot be realized over any distribution, and therefore cannot be accessed by the adversary at all. For simplicity, we'll also restrict our attention to (multi-class) classification where the label space $Y=[m]$, and recall that the loss of any undefined point is always $1$.
\begin{definition}[(Realizable) Distribution-Family Partial PAC Learning]
A partial class $(\dist, X, H)$ is PAC-learnable in the realizable setting if there exists an algorithm $\algo$ and function $n(\varepsilon,\delta)$ such that for all $\varepsilon,\delta > 0$ and distributions $D$ over $X\times Y$ satisfying:
\begin{enumerate}
    \item The marginal $D_X \in \dist_H$,
    \item $\min_{h \in H} \text{err}_{D}(h) = 0$,
\end{enumerate}
then:
\[
\Pr_{S \sim D^{n(\varepsilon,\delta)}}[\text{err}_{D}(\algo(S)) > \varepsilon] \leq \delta,
\]
where the error $\text{err}_{D}(h)$ is standard classification error:
\[
\text{err}_{D}(h) = \Pr_{(x,y) \sim D}[h(x) \neq y].
\]
\end{definition}
Agnostic learnability is defined analogously, but since the adversary is unrestricted, there is no need to move to the closure of $\dist$.
\begin{definition}[(Agnostic) Distribution-Family Partial PAC Learning]
A partial class $(\dist, X, H)$ is PAC-learnable in the agnostic setting if there exists an algorithm $\algo$ and function $n(\varepsilon,\delta)$ such that for all $\varepsilon,\delta > 0$ and distributions $D$ over $X\times Y$ satisfying $D_X \in \dist$, then:
\[
\Pr_{S \sim D^{n(\varepsilon,\delta)}}[\text{err}_{D}(\algo(S)) > OPT + \varepsilon] \leq \delta.
\]
\end{definition}

The issue with our standard reduction strategy for partial functions is that in the agnostic model, the adversary's marginal distribution over $X$ might have support outside of $\text{supp}(h_{OPT})$, which causes \textsc{LearningToCover} to lose its guarantee of outputting a non-uniform cover. This can be dealt with by a variant of our subsampling technique. If we run \textsc{LearningToCover} over all subsamples of the unlabeled sample $S_U$, one of these subsamples must match the support of $h_{OPT}$. This is in fact the same strategy used for adversarial robustness in \Cref{sec:robust}, but we will include the algorithm again to make this section self-contained. 

\begin{algorithm}[H]
\SetAlgoLined
\textbf{Input:} Realizable Robust PAC-Learner $\algo$
\\
\textbf{Algorithm:}
\begin{enumerate}[leftmargin=*]
    \item Draw an unlabeled sample $S_U$, and labeled sample $S_L$.
    \item Run $\algo$ over all possible subsets and labelings of $S_U$ to get:
    \[
    C(S) \coloneqq \{ \algo(S,h(S))~|~ S \subseteq S_U, h \in H|_S\}.
    \]
    \item \textbf{Return} the hypothesis in $C(S)$ with lowest empirical error over $S_L$.
\end{enumerate}
 \caption{Agnostic to Realizable Reduction Partial PAC Setting}
 \label{alg:partial}
\end{algorithm}

\begin{theorem}\label{thm:partial}
If $(\dist,X,H)$ is a realizably PAC-learnable partial class with sample complexity $n(\varepsilon,\delta)$, then \Cref{alg:partial} agnostically learns $(\dist,X,H)$ in
\[
m_U(\varepsilon,\delta) \leq O\left(\max_{\mu \in [0,1-\varepsilon]}\left\{\frac{n(\varepsilon/(2(1-\mu)),\delta/3)}{1-\mu} \right\}\right)
\]
unlabeled samples, and
\[
m_L(\varepsilon,\delta) \leq O\left( \frac{m_U(\varepsilon,\delta) + \log(1/\delta)}{\varepsilon^2}\right)
\]
labeled samples.
\end{theorem}
\begin{proof}
The proof is essentially the same as for \Cref{thm:robust}, but we repeat it here for completeness. As always, it is enough to prove that $C(S)$ (from \Cref{alg:partial}) contains a hypothesis $h'$ with error at most $OPT+\varepsilon/2$. The key issue with our standard reduction is that the optimal hypothesis $h_{OPT}$ may be undefined on certain examples in the unlabeled sample $S_U$. By running over all subsamples of $S_U$, we in essence simulate pulling samples only from the support of $h_{OPT}$, which is enough to get the desired guarantee.

More formally, let $D|_{h_{OPT}}$ be the restriction of $D$ to $\text{supp}(h_{OPT})$, and $\bar{D}|_{h_{OPT}}$ the restriction to its complement $X \setminus \text{supp}(h_{OPT})$. The idea is to decompose our analysis into two separate parts over $D|_{h_{OPT}}$ and $\bar{D}|_{h_{OPT}}$. With this in mind, let $\mu^*$ denote the mass of $D_X$ on the undefined portion of $h_{OPT}$, and let $OPT'$ denote the error of $h_{OPT}$ over $D|_{h_{OPT}}$. Since we have restricted our attention to classification, notice that we can decompose $OPT$ as:
\begin{align*}
\text{err}_{D}(h_{OPT}) &= \underset{(x,y) \sim D}{\Pr}\left[h_{OPT}(x) \neq y \right]\\
    &= \mu^*\underset{(x,y) \sim \bar{D}|_{h_{OPT}}}{\Pr}\left[h_{OPT}(x) \neq y \right] + (1-\mu^*)\underset{(x,y) \sim D|_{h_{OPT}}}{\Pr}\left[h_{OPT}(x) \neq y \right]\\
    &= \mu^* + (1-\mu^*)OPT'.
\end{align*}
We'd like to prove that $C(S)$ contains a hypothesis $h$ within $\varepsilon/2$ error of optimal. We claim it is sufficient to show that $C(S)$ contains a hypothesis within $\varepsilon/(2(1-\mu^*))$ classification distance of $h_{OPT}$, since:
\begin{align*}
\text{err}_{D}(h) &= \mu^*\underset{(x,y) \sim \bar{D}|_{h_{OPT}}}{\mathbb{E}}\left[h(x) \neq y \right] + (1-\mu^*)\underset{(x,y) \sim D|_{h_{OPT}}}{\mathbb{E}}\left[h(x) \neq y \right]\\
&\leq \mu^* + (1-\mu^*)\underset{(x,y) \sim D|_{h_{OPT}}}{\mathbb{E}}\left[h(x) \neq y \right]\\
&\leq \mu^* + (1-\mu^*)\left(\underset{(x,y) \sim D|_{h_{OPT}}}{\mathbb{E}}\left[h_{OPT}(x) \neq y \right] + \frac{\varepsilon}{2(1-\mu^*)}  \right)\\
&=\mu^* + (1-\mu^*)OPT' + \frac{\varepsilon}{2} \\
&= OPT + \varepsilon/2,
\end{align*}
where the third line follows from the fact that $h'$ and $h_{OPT}$ only differ on a $\frac{\varepsilon}{2(1-\mu^*)}$ fraction of inputs over $D|_{h_{OPT}}$.

It is left to argue that $C(S)$ contains such a hypothesis $h$. Recall that on a labeled sample $(S,h(S)) \sim D|_{h_{OPT}} \times h_{OPT}$ of size $n(\varepsilon/(2(1-\mu^*)),\delta/3)$, \textsc{LearningToCover} will contain $h$ that is $\varepsilon/(2(1-\mu^*))$-close to $h_{OPT}$ in classification error over $D|_{h_{OPT}}$ with probability at least $1-\delta/3$. The idea is then to draw a large enough unlabeled sample such that with probability at least $1-\delta/3$, the restriction of the sample to $D|_{h_{OPT}}$ is at least this size (since we run over every subsample, we will always hit this restriction). By a Chernoff bound, it is enough to draw $c_1\frac{n(\varepsilon/(2(1-\mu^*)),\delta/3)}{1-\mu^*}$ points to achieve this for some large enough constant $c_1>0$.\footnote{As in \Cref{thm:robust}, we've assumed for simplicity that $n(\varepsilon,\delta) \geq \Omega(\log(1/\delta))$. This assumption can be removed by including an extra additive factor of $\log(1/\delta)$.} Since we do not know $\mu^*$, we'll need to draw $c_1 \max_{\mu \in [0,1-\varepsilon]}\left\{\frac{n(\varepsilon/(1-\mu),\delta/3)}{1-\mu} \right\}$ points to ensure this property holds (if $\mu^* \geq 1-\varepsilon$, note that any hypothesis gives a valid solution). By a union bound we have that this overall process succeeds with probability at least $1-2\delta/3$, and outputting the hypothesis in $C(S)$ with the lowest empirical risk then succeeds with probability $1-\delta$ as desired.
\end{proof}
Like \Cref{thm:robust}, \Cref{thm:partial} can be extended to many of the generic property generalization results in the main body, including approximate pseudometric loss, malicious noise, and semi-private learning, though it may experience some degradation of parameters (e.g.\ $c$-agnostic rather than truly agnostic learning) depending on how the loss of ``$*$'' values are formalized in these settings. Again this comes at the cost of an additional factor of $\varepsilon$ over known reductions based on sample compression in the finite VC regime \cite{alon2021theory}.


\section{Uniform Stability}\label{sec:stable}
Uniform stability, originally introduced by Bousquet and Elisseeff \cite{bousquet2002stability}, is a useful algorithmic property that is closely tied to both generalization and privacy. Informally, an algorithm $\algo$ is said to be uniformly stable if for all elements $x \in X$, the probability that $\algo$ changes its output on $x$ over neighboring datasets is small.
\begin{definition}[Uniform Stability]
    A learning algorithm is said to be $\alpha$-uniformly stable if for all neighboring inputs $S,S'$ which differ on a single example, all $x \in X$, and all $y \in Y$: 
    \[
    \Pr[\algo(S)(x) = y] \leq \Pr[\algo(S')(x) = y] + \alpha.
    \]
\end{definition}
Uniform stability can also be thought of as a variant of \textit{private prediction} \cite{dwork2018privacy}, which protects against adversaries who have restricted access to a model only through prediction responses on individual points (this is often the case in practice since it is common to release APIs with query access rather than full models). Like semi-privacy, this definition has the benefit of maintaining practicality in a reasonable range of circumstances while weakening the stringent requirements of standard private learning. Indeed, it is well known that in the distribution-free classification setting, uniformly stable learning and private prediction are both possible for any class with finite VC dimension \cite{shalev2010learnability,dwork2018privacy,dagan2020pac}. Unsurprisingly, these previous works (at least those working in the agnostic model), rely on uniform convergence and uniform covers. We'll show these can be replaced with a variant of our standard reduction. The argument is otherwise similar to the proof in \cite{dagan2020pac}.

\begin{theorem}\label{thm:stable}
    Let $(\dist,X,H)$ be a realizably learnable class with sample complexity $n(\varepsilon,\delta)$. Then there exists an $\alpha$-uniformly stable, $\alpha$-semi private algorithm that agnostically learns $(\dist,X,H)$ in only
    \[
    m_U(\varepsilon,\delta,\alpha) \leq O\left( \frac{n(\varepsilon/2,\delta/2)}{\alpha} \right)
    \]
    unlabeled samples, and
    \[
    m_L(\varepsilon,\delta,\alpha) \leq O\left( \frac{\log\left(\Pi_H(n(\varepsilon,\delta)) \right)}{\min\{\alpha\varepsilon,\varepsilon^2\}} \right)
    \]
    labeled samples.
\end{theorem}
\begin{proof}
The proof boils down to a standard subsampling trick first noted by \cite{shalev2010learnability}. Instead of drawing our standard unlabeled sample of size $n(\varepsilon/2,\delta/2)$, we draw a sample of size $\frac{n(\varepsilon/2,\delta/2)}{2\alpha}$ and run \textsc{LearningToCover} over a random $\alpha/2$ fraction of the sample. This ensures that swapping out any individual sample can only effect the result with probability $\alpha/2$. Since this subsample is of size $n(\varepsilon/2,\delta/2)$, \textsc{LearningToCover} keeps its standard guarantees and the output $C(S_U)$ has a hypothesis within $\varepsilon/2$ of optimal with probability $1-\delta/2$. We can now apply the exponential mechanism with privacy parameter $\alpha/4$, which ensures the algorithm is $\alpha/2$-uniformly stable with respect to the labeled sample as well. The sample complexity bounds come from standard analysis of the exponential mechanism and the size of $C(S_U)$. Semi-privacy comes for free due to our use of the exponential mechanism.
\end{proof}
As in previous sections, \Cref{thm:stable} can be extended to any of the generic property generalization results in the main body, including for instance $c$-approximate pseudometric loss, malicious noise, and robustness to covariate shift. In the VC setting, the complexity essentially matches the best known bounds in the literature, which are given by a similar reduction using uniform convergence \cite{dagan2020pac}.

\section{Statistical Query Model}\label{sec:SQ}
Kearns' \cite{kearns1998efficient} statistical query model is a popular modification of PAC learning where the sample oracle is replaced with the ability to ask noisy statistical questions about the data.
\begin{definition}[Realizable SQ-learning]
    Given a distribution $D$ over $X$ and $h \in H$, let $STAT(D,h)$ be an oracle which upon input of a function $\psi:  X \times Y \to [-1,1]$ and tolerance $\tau \in \Rplus$ may output any estimate of the expectation of $\psi$ up to $\tau$ error, that is:
    \[
    STAT(D,h)(\psi,\tau) \in \mathbb{E}_{x \sim D}[\psi(x,h(x))] \pm \tau.
    \]
    We call a class $(\dist,X,H,\ell)$ SQ-learnable if for all $\varepsilon>0$, there exists some tolerance $\tau = \tau(\varepsilon)$, query complexity $n(\varepsilon,\tau)$, and an algorithm $\algo$ such that for all $D \in \dist$ and $h \in H$, $\algo$ achieves $\varepsilon$ error in at most $n(\varepsilon,\tau)$ oracle calls to $STAT(D,h)$ with tolerance at worst $\tau$.\footnote{While we generally think of $\tau$ as being at worst polynomial in $\varepsilon$, this is not strictly necessary for the model.}
\end{definition}
Agnostic learning is then defined analogously where $D,h$ is replaced with a generic distribution over $X \times Y$ whose marginal lies in $\dist$.
We can use a basic form of discretization to prove property generalization in the SQ model.
\begin{theorem}\label{thm:SQ}
    Let $\ell$ be a $c$-approximate pseudometric and $(\dist,X,H,\ell)$ a realizably SQ-learnable class with query complexity $n(\varepsilon,\tau)$. Then $(\dist,X,H,\ell)$ is $c$-agnostically SQ-learnable up to $\varepsilon+\tau$ error in $(1/\tau)^{n(\varepsilon,\tau)}$ statistical queries of tolerance at worst $\tau$.
\end{theorem}
\begin{proof}
The idea is similar to our discretization in \Cref{thm:approximate pseudometric}. The realizable SQ-learner $\algo$ makes some finite $n(\varepsilon,\tau)$ queries. Let $C_\algo$ denote the set of outputs of $\algo$ when fed every possible combination of responses from the discretized set $\{-1,-1+2\tau,\ldots,1-2\tau,1\}$. For every $D \in \dist$ and $h \in H$, one of these combinations must be a valid query response in the realizable model, so $C_\algo$ covers $(\dist,X,H,\ell)$. By the same arguments of \Cref{thm:approximate pseudometric}, $C_\algo$ must contains a hypothesis with error at most $c \cdot OPT + \varepsilon$. Since we can directly compute the loss of every element in $C_\algo$ up to $\tau$ error in the SQ model simply by querying the loss function, this gives the desired result in $|C_\algo| = (1/\tau)^{n(\varepsilon,\tau)}$ queries.
\end{proof}
We note that while our reduction in this model experiences exponential blowup in the number of queries, this should really be thought of as corresponding to a blow up in \textit{run-time} instead of ``sample complexity'' in the standard sense (which corresponds more closely to $\tau$).

\section{Fairness}\label{sec:fair}
Recent years have seen rising interest in an algorithmic property called \textit{fairness}. Informally, fairness tries to tackle the issue that ``well-performing'' classifiers in the standard sense may actually be discriminatory against certain individuals or subgroups. We will consider a form of fair leaning introduced by Rothblum and Yona \cite{rothblum2018probably} called Probably Approximately Correct and Fair (PACF) learning. Their definition is based off of a notion of fairness that ensures that similar individuals are treated similarly with respect to a fixed metric.
\begin{definition}[Metric Fairness]
    Let $d: X \times X \to \Rplus$ be a similarity measure on $X$ and $D$ a distribution over $X$. A classifier $h:X \to Y_{\text{out}}$ is called $(\alpha,\gamma)$-fair with respect to $d$ and $D$ if h acts similarly on most similar individuals:
    \[
    \Pr_{x,x' \sim D}[|h(x)-h(x')| > d(x,x') + \gamma] \leq \alpha.
    \]
We note that the output space $Y_{\text{out}}$ may differ from the label space $Y$ in general learning problems.
\end{definition}
In fact, this definition only really makes sense when the output classifier $h$ is allowed to be real-valued (as this allows for some flexibility in the $|h(x)-h(x')|$ term). As such, when considering settings such as binary classification where $Y=\{0,1\}$ is discrete, Rothblum and Yona's \cite{rothblum2018probably} initial formalization considers returning \textit{probabilistic} classifiers with $Y_{\text{out}} =[0,1]$. Here $h(x)=y \in [0,1]$ is taken to be the probability of the label being $1$. The error of a probabilistic classifier $h$ with respect to any distribution $D$ over $X \times \{0,1\}$ is then given by its expected $\ell_1$ distance:
\[
err_D(h) = \underset{(x,y) \sim D}{\mathbb{E}}[|h(x)-y|].
\]
For simplicity we'll focus in this section on this same regime extended to the distribution-family model. 

In broad strokes, the goal of Fair PAC learning is to output a fair classifier satisfying standard PAC guarantees. Practically this requires a few modifications. First, since there may be no fair classifier satisfying these guarantees, we will only require our output to be as good as the best fair classifier. Second, we will actually allow some slack in the fairness parameters, which Rothblum and Yona \cite{rothblum2018probably} show is a practical way to ensure that fair learnability remains possible across a broad range of classes.
\begin{definition}[PACF-learning\protect\footnotemark\  \cite{rothblum2018probably}]\footnotetext{We note that our presentation of this definition differs slightly from \cite{rothblum2018probably}. Their $(\alpha,\gamma)$-PACF-learnability formally corresponds to $(\alpha-\varepsilon_\alpha,\gamma-\varepsilon_\gamma)$-PACF-learnability in our version.} 
We say $(\dist, X, H)$ is (agnostically) $(\alpha,\gamma)$-PACF-learnable with respect to a similarity metric $d: X \times X \to Y$ if there exists an algorithm $\algo$ and function $n=n(\varepsilon,\varepsilon_\alpha,\varepsilon_\gamma,\delta)$ such that for all $\varepsilon,\varepsilon_\alpha,\varepsilon_\gamma,\delta > 0,$ and distributions $D$ over $X \times Y$ such that $D_X \in \dist$, $\algo(S)$ satisfies the following guarantees with probability $1-\delta$ over samples $S$ of size $n$:
\begin{enumerate}
    \item $\algo(S)$ is accurate:
        \[
            \text{err}_{D,\ell}(\algo(S)) \leq OPT_{\alpha,\gamma} + \varepsilon
        \]
    \item $\algo(S)$ is $(\alpha+\varepsilon_\alpha,\gamma+\varepsilon_\gamma)$-fair.
\end{enumerate}
Here $OPT_{\alpha,\gamma}$ is the optimal error of any $(\alpha,\gamma)$-fair classifier, that is:
\[
OPT_{a,b} \coloneqq \min_{h \in H^d_{D_X,a,b}}\{err_{D,\ell}(h)\},
\]
and
\[
H^d_{D_X,a,b} = \{h \in H: h \text{ is $(a,b)$-fair with respect to $d$ and $D_X$}\}
\]
\end{definition}
Realizable learnability is defined similarly, where the adversary is constrained to picking distributions which have $0$ error with respect to some $(\alpha,\gamma)$-fair classifier in $H$. We show that property generalization holds for the PACF model.
\begin{theorem}[Agnostic $\to$ Realizable (PACF Setting)]\label{thm:fair}
    Let $(\dist,X,H)$ be any class that is realizably $(\alpha,\gamma)$-PACF learnable with sample complexity $n(\varepsilon,\varepsilon_\alpha,\varepsilon_\gamma,\delta)$. Then $(\dist,X,H)$ is agnostically $(\alpha,\gamma)$-fair-PAC learnable in only
    \[
    m_U(\varepsilon,\varepsilon_\alpha,\varepsilon_\gamma,\delta) \leq n(\varepsilon/2,\varepsilon_\alpha,\varepsilon_\gamma,\delta/2)
    \]
    unlabeled samples, and
    \[
    m_L(\varepsilon,\varepsilon_\alpha,\varepsilon_\gamma,\delta) \leq O\left(\frac{\log(\Pi_H(n(\varepsilon/2,\varepsilon_\alpha,\varepsilon_\gamma,\delta/2)))+\log(1/\delta)}{\varepsilon^2}\right)
    \]
    labeled samples.
\end{theorem}
\begin{proof}
The key observation is that the definition of fairness depends only on the classifier $h$ and the \textit{marginal distribution} $D_X$. Let $h_{OPT}$ be the hypothesis achieving the minimum error over $H_{D_X,\alpha,\gamma}$. By the above observation, with probability $1-\delta$ the hypothesis set $C(S_U)$ returned by \textsc{LearningToCover} contains an \textit{$(\alpha+\varepsilon_\alpha,\gamma+\varepsilon_\gamma)$-fair} $h$ satisfying:
\[
\mathbb{E}_{x \sim D_X}[|h(x) - h_{OPT}(x)|] \leq \varepsilon/2.
\]
Since $\ell_1$ error is a metric (and therefore satisfies the triangle inequality), we can use our argument for $c$-pseudometric loss functions from \Cref{thm:approximate pseudometric} to argue that choosing the lowest empirical risk \textit{$(\alpha+\varepsilon_\alpha,\gamma+\varepsilon_\gamma)$-fair} classifier in $C(S_U)$ with respect to a sufficiently large labeled sample $S_L$ gives the desired learner.
\end{proof}
With care, this result can be extended to a broader range of loss functions as well as to other finitely-satisfiable properties covered in this work.

\section{Notions of Coverability}\label{sec:cover}
In this section we discuss the connection between non-uniform covers and several previous notions of coverability used in various learning applications. For simplicity, we'll restrict our attention to covering with respect to standard classification distance, that is given a distribution $D$ and hypotheses $h$ and $h'$ over some instance space $X$:
\[
d_D(h,h') = \Pr_{x \sim D}[h(x) \neq h'(x)].
\]
To start, let's recall the basic notion of an $\varepsilon$-cover specified to this measure for simplicity.

\begin{definition}[$\varepsilon$-cover]
Let $X$ be an instance space, $Y$ a label space, and let $L_{X,Y}$ denote the set of all labelings $c: X \to Y$.
A set $C \subset L_{X,Y}$ is said to form an $\varepsilon$-cover for $(D,X,H)$  
if for every hypothesis $h \in H$, there exists $c\in C$ such that
\[
d_D(c,h) \leq \varepsilon\, .
\]
$C$ is called \textbf{proper} if $C \subset H$.
\end{definition}
Finite $\varepsilon$-covers are exceedingly useful in learning theory. As discussed in \Cref{sec:related-work}, a common strategy in the literature is to use unlabeled samples to construct an $\varepsilon$-cover with high probability \cite{balcan2010discriminative,hanneke2015minimax,alon2019limits,bassily2020private}. This results in a distribution over potential covers we call a uniform $(\varepsilon,\delta)$-cover.
\begin{definition}[Uniform $(\varepsilon,\delta)$-cover]
Let $X$ be an instance space, $Y$ a label space, and let $L_{X,Y}$ denote the set of all labelings $c: X \to Y$.
A distribution $D_C$ over the power set $P(L_{X,Y})$ is said to form a \textbf{uniform} $(\varepsilon, \delta)$-cover for $(D,X,H)$ if:
\[
\Pr_{C \sim D_C} [C ~\text{is an $\varepsilon$-cover for}~ (D,X,H)] \geq 1-\delta.
\]
$D_C$ is called \textbf{proper} if its support lies entirely in $H$.
\end{definition}

In this work, we introduce a weaker non-uniform variant of this notion where each $h$ has an individual guarantee of being covered by the distribution, but it is not necessarily the case that a sample will cover all $h \in H$ simultaneously. 

\begin{definition}[Non-Uniform $(\varepsilon, \delta)$-cover]
Let $X$ be an instance space, $Y$ a label space, and let $L_{X,Y}$ denote the set of all labelings $c: X \to Y$.
A distribution $D_C$ over the power set $P(L_{X,Y})$ is said to form a \textbf{non-uniform} $(\varepsilon, \delta)$-cover for $(D,X,H)$  
if for every fixed hypothesis $h \in H$, 
\[
\Pr_{C \sim D_C} [C ~\text{is an $\varepsilon$-cover for}~ (D,X,\{h\})] \geq 1-\delta.
\]
$D_C$ is called \textbf{proper} if its support lies entirely in $H$.
\end{definition}


In the context of learning, we are usually interested not just in the existence of these covers, but in the more challenging problem of constructing them from a small number of unlabeled samples. In other words, given a class $(\mathscr{D},X,H)$, we'd like to know how many unlabeled samples from an adversarially chosen distribution $D \in \dist$ are necessary to build a uniform (or non-uniform) $(\varepsilon,\delta)$-cover for $(D,X,H)$. In \Cref{sec:privacy}, we saw that the ability to construct a non-uniform $(\varepsilon,\delta)$-cover from $O\left(\frac{\log(1/\delta)}{\varepsilon}\right)$ samples was crucial to give a semi-private learner with optimal public sample complexity. This improved over recent work of Alon, Bassily, and Moran (ABM) \cite{alon2019limits}, who showed that it is possible to build a uniform $(\varepsilon,\delta)$-cover in $O\left(\frac{\log(1/\varepsilon)+\log(1/\delta)}{\varepsilon}\right)$ samples. 

It is interesting to ask whether non-uniformity is really necessary here, or whether ABM's analysis is simply sub-optimal. We'll show that the former is true, at least in the proper distribution-family setting: the $\log(1/\varepsilon)$ gap between these models is necessary and uniform covers cannot be used to build optimal semi-private learners. 

\begin{theorem}[Separation of Uniform and Non-Uniform Covers]\label{thm:separation}
There exists an instance space $X$, hypothesis class $H$, and family of distributions $\dist$ such that for any sufficiently small $\varepsilon > 0$, the following statements holds:
\begin{enumerate}
    \item Any algorithm which returns a finite proper uniform $(\varepsilon,1/3)$-cover for $(\dist,X,H)$ requires at least $\Omega(1/\varepsilon \cdot  \log(1/\varepsilon))$ samples.
    \item There exists an algorithm which returns a finite proper non-uniform $(\varepsilon, \delta)$-cover for $(\dist,X,H)$ in $O(\log(1/\delta)/\varepsilon)$ samples.
\end{enumerate}
\end{theorem}
\begin{proof}
Let the instance space $X = \N$ and $H$ be the class of indicators along with the all $0$'s function, that is $H = \{h_i: i \in \N\} \cup \{h_{0}\}$ where $h_i(x) = \mathbbm{1}\{x = i\}$ and $h_{0}$ is $0$ everywhere. We consider the family of distribution $\dist =  \{\dist_{n, k}\}_{n,k > 0}$ given by $k$-sets of $[n]$ where  
\[
\dist_{n,k} = \{\unif(T): T\subset [n] ~\text{and}~ |T| = k\},
\] where $\unif(T)$ is a uniform distribution over $T$.

We start with the first claim, that building a bounded uniform $(\varepsilon,1/2)$-cover needs at least $\Omega(1/\varepsilon \log(1/\varepsilon))$ samples. More formally, for any error parameter  $\varepsilon>0$ and size bound $m=m(\varepsilon) \in \mathbb{N}$, let $k = \lfloor 1/(2\epsilon) \rfloor$. We will show that for any algorithm $\algo$ on $k\log(k)$ samples that outputs at most $m$ hypotheses, $\algo$ must fail to output an $\varepsilon$-cover with probability at least $1/2$.

Let $n \gg m,k$ be some natural number to be fixed later and consider the family of distributions $\dist_{n,k}$. By Yao's minimax principle, it is sufficient to show that there exists a distribution over the elements in $\dist_{n,k}$ such that any deterministic algorithm over $k\log(k)$ samples outputting a set of (at most) $m$ hypotheses fails to give a proper $\varepsilon$-cover with probability $1/2$. We claim that taking the uniform distribution over $\dist_{n,k}$ suffices. To formalize this, it is useful to observe the following claim.
\begin{claim}\label{claim:covering}
Any subset of hypotheses $C \subset H$ of size $m$ can be a proper $\varepsilon$-cover
for $H$ under at most ${m \choose k}$ distributions in $\dist_{n,k}$.
\end{claim}

Let's prove the result under this assumption. The key observation is that by standard lower bounds on the coupon collector problem, a sample $S$ of $k\log(k)$ points from any $\unif(T) \in \dist_{n,k}$ will not include $\unif(T)$'s entire support with probability at least $1/2$. With this in mind, assume that the input sample $S$ contains only $\text{supp}(S) = j < k = \text{supp}(\unif(T))$ elements. As a result, there are ${n-j \choose k-j}$ consistent distributions with $S$, and by \Cref{claim:covering}, $\algo(S)$ is a proper $\varepsilon$-cover for at most ${m \choose k}$ of them. Since $S$ is equally likely to have been sampled from any of these distributions, the probability that $\algo(S)$ is a proper $\varepsilon$-cover is at most:
\[
\Pr[\algo \text{ fails given } \text{supp}(S)=j<k] \geq \frac{{n-j \choose k-j} - {m \choose k}}{{n-j \choose k-j}}.
\]
Taking $n$ sufficiently larger than $m$ and $k$, we can make this probability as close to $1$ as desired for any $0 < j < k$. Finally, since samples of this form occur with probability at least $1/2$, the algorithm fails with probability at least $1/3$ as desired. It is left to prove \Cref{claim:covering}.
\begin{proof}[Proof of Claim \ref{claim:covering}]
Notice that for any distribution $\unif(T) \in \dist_{n,k}$, any $i \in T$ and any $j \neq i$, $d_{\unif(T)}(h_i, h_j) > 2\varepsilon$. Let $C$ be any proper $\varepsilon$-cover of $H$ under distribution $\unif(T)$. Then, by the above argument, it must contain $\{h_i : i \in T\}$. Since $|T| = k$, $C$ can be a proper $\varepsilon$-cover of $H$ under at most $|C| \choose k$ distributions in $\dist_{n,k}$. 
\end{proof}

We now move to proving that a proper \emph{non-uniform} $(\varepsilon, \delta)$-cover can be built in only $O(\log(1/\delta)/\varepsilon)$ samples. 
This follows from the fact that for any $n \geq k$ and distribution $\unif(T) \in \dist_{n,k}$, each $i \in T$ is in the random sample $S$ with probability $1-\delta$. Since each $h_j$ for $j \notin T$ is covered by $h_0$, outputting $\{h_i : i \in S\} \cup \{h_{0}\}$ generates a proper non-uniform $(\varepsilon, \delta)$-cover.
\end{proof}
The construction in \Cref{thm:separation} can easily be modified to give a class with the same gap which is not privately learnable (say by embedding a single copy of a threshold over $[0,1]$). Since any such class requires at least $\Omega(\frac{1}{\varepsilon})$ public samples to semi-privately learn by \Cref{thm:abm-lower},\footnote{Alon, Bassily, and Moran only state this result for the distribution-free setting, but it holds in the distribution-family model as well.} \Cref{thm:separation} then provides a separation between using uniform and non-uniform covers in semi-private learning: the former provably requires an extra log factor, while the latter matches the lower bound exactly. Unfortunately, our proof of this result only holds in the proper setting, as \Cref{claim:covering} fails when improper hypotheses are allowed. We conjecture that this is not an inherent barrier: the separation should continue to hold in the improper case, albeit with some different analysis.

We have now seen a weak separation between uniform and non-uniform covers, but one might reasonably wonder whether a much stronger separation is possible. In particular, all previous constructions of uniform covers use uniform convergence, but there exist simple examples of learnable classes in the distribution-family model that fail this property: do such classes provide an example of objects which are non-uniformly coverable but not uniformly coverable? Surprisingly, the answer is no! It turns out that an algorithm for non-uniform covering can always be used to construct a uniform covering without too much overhead. Moreover, we'll see that the $\log(1/\varepsilon)$ gap is tight when $(X,H)$ has finite VC dimension.

To prove this, it will actually be useful to make a brief aside and introduce another closely related notion of covering called \textit{fractional covers}. These objects are essentially a form of non-uniform covering which output a single hypothesis instead of a set of them.
\begin{definition}[Fractional cover]
Let $X$ be an instance space, $Y$ a label space, and let $L_{X,Y}$ denote the set of all labelings $c: X \to Y$.
A distribution $D_C$ over $L_{X,Y}$ is said to form a fractional $(\varepsilon, p)$-cover for a hypothesis class $H$ for $(D,X,H)$ if for any fixed $h \in H$, a sample from $D_C$ covers $h$ with probability $p$:
\[
\Pr_{c \sim D_C} [d(c,h) \leq \varepsilon] \geq p.
\]
\end{definition}
Fractional covers are closely connected to non-uniform covers. In fact, one can easily move between the two by sampling or subsampling.
\begin{proposition}[Non-uniform cover $\iff$ Fractional cover]\label{lemma:frac-non-uniform}
Let $(D,X,H)$ be any class, $C_{\text{frac}}$ a fractional $(\varepsilon,p)$-cover, and $C_{\text{n-u}}$ a non-uniform $(\varepsilon,1/2)$-cover. Then the following hold:
\begin{enumerate}
    \item Drawing $\log_{1/(1-p)}(1/\delta)$ samples from $C_{\text{frac}}$ gives a non-uniform $(\varepsilon,\delta)$-cover.
    \item Choosing a random hypothesis from $C_{\text{n-u}}$ gives a fractional $(\varepsilon,1/2|C|)$-cover.
\end{enumerate}
\end{proposition}
\begin{proof}
Both statements are essentially immediate from definition. For any fixed $h \in H$, if we draw $M$ samples from $C_{\text{frac}}$, the probability we fail to cover $h$ is $(1-p)^{M}$, so setting $M=\log_{1/(1-p)}(1/\delta)$ gives the desired non-uniform cover. On the other hand, for any fixed $h \in H$, a sample from $C \sim C_{\text{n-u}}$ contains $c$ $\varepsilon$-close to $h$ with probability $1/2$. Outputting a uniformly random element of $C$ then gives an element within $\varepsilon$ of $h$ with probability $1/2|C|$ as desired.
\end{proof}
It will also be useful to note a classical relation between covers and fractional covers. 
\begin{lemma}\label{lemma:pack-cover}
If there exists a fractional $(\varepsilon,p)$-cover for $(D,X,H)$, then there exists a $2\varepsilon$-cover of size $1/p$.
\end{lemma}
\begin{proof}
This follows from classical packing-covering duality. The existence of a fractional $(\varepsilon,p)$-cover implies there cannot exist a $2\varepsilon$-packing of size greater than $1/p$ (that is, a set of more than $1/p$ hypotheses in $H$ that are pairwise $2\varepsilon$-separated with respect to $D$). By packing-covering duality, this implies the existence of a $2\varepsilon$-cover of size $\lceil 1/p \rceil + 1$.
\end{proof}
With this in hand, let's show that uniform covers can be constructed for any realizably learnable class, regardless of whether we have uniform convergence.

\begin{theorem}[Realizable learning $\to$ Uniform cover]\label{thm:real-to-uniform}
Let $(\dist,X,H)$ be realizably PAC-learnable with sample complexity $n(\varepsilon,\delta)$. Then it is possible to construct a uniform $(\varepsilon,\delta)$-cover for $(\dist,X,H)$ in $n(\varepsilon/2,\delta')$ samples where $\delta'=O\left(\frac{\delta}{\Pi_H(n(\varepsilon/2,1/2))}\right)$.
\end{theorem}
\begin{proof}
We'll start by proving a slightly more general fact. If for every $D \in \dist$, $(D,X,H)$ has a proper $(\varepsilon/2)$-cover $C_D$ of size at most $C=C(\varepsilon/2)$, then it is possible to construct a uniform $(\varepsilon,\delta)$-cover in $n(\varepsilon/2,\delta/C)$ samples. This is essentially immediate from \Cref{lemma:cover}, which states that running \textsc{LearningToCover} over a sample of size $n(\varepsilon/2,\delta/C)$ gives a non-uniform $(\varepsilon/2,\delta/C)$-cover. Union bounding over $C_D$ then gives that a sample from the non-uniform cover $(\varepsilon/2)$-covers $C_D$ with probability at least $1-\delta$. Since $C_D$ is itself an $(\varepsilon/2)$-cover, this implies that the entire class $H$ $\varepsilon$-covered by the sample with probability at least $1-\delta$ as desired.

It remains to show that for every $D \in \dist$, $(D,X,H)$ has a proper $(\varepsilon/2)$-cover of size $O(\Pi_H(n(\varepsilon/2,1/2)))$. This follows from combining \Cref{lemma:frac-non-uniform} and \Cref{lemma:pack-cover}. In particular, \Cref{lemma:cover} implies that running \textsc{LearningToCover} over a sample of size $n(\varepsilon/2,1/2)$ produces a non-uniform $(\varepsilon/2,1/2)$-cover of size at most $\Pi_H(n(\varepsilon/2,1/2))$. \Cref{lemma:frac-non-uniform} states that subsampling from this cover gives a fractional $(\varepsilon/2,1/(2\Pi_H(n(\varepsilon/2,1/2))))$-cover, which in turn implies the existence of a $(\varepsilon/2)$-cover of size $O(\Pi_H(n(\varepsilon/2,1/2)))$ as desired. We note that this last argument is similar to an observation made in Benedek and Itai's \cite{benedek1991learnability} seminal work on the distribution-dependent model.
\end{proof}
When $(X,H)$ has finite VC-dimension $d$, note that \Cref{thm:real-to-uniform} exactly matches the lower bound exhibited in \Cref{thm:separation} as the required number of samples for a uniform $(\varepsilon,\delta)$-cover becomes:
\[
n(\varepsilon/2,\delta') \leq O\left(\frac{d\log(1/\varepsilon) + \log(1/\delta)}{\varepsilon}\right).
\]
This also matches the bound given by ABM \cite{alon2019limits} using uniform convergence. 
\section*{Acknowledgements}
The authors would like to thank Shay Moran, Russell Impagliazzo, Omar Montasser, and Avrim Blum for enlightening discussions. We also thank anonymous referees for constructive feedback, and especially for pointing out the notion of probabilistic representations and that prior work discussed in \Cref{sec:related-work} falls under the general framework of our reduction.

\clearpage
\printbibliography

\end{document}